\documentclass{article}

\PassOptionsToPackage{numbers, compress}{natbib}
     \usepackage[final]{neurips_2020}

\usepackage{microtype}
\usepackage[pdftex]{graphicx}
\usepackage{graphicx}
\usepackage{subfigure}
\usepackage{float}
\usepackage{booktabs} 
\usepackage{amsmath}
\usepackage{multicol}
\usepackage{amssymb}
\usepackage{amsmath,bm}
\usepackage{wrapfig}
\usepackage{enumitem}
\usepackage{amsthm}
\usepackage{algorithm}
\usepackage{tabularx}
\usepackage[noend]{algpseudocode}
\newtheorem{theorem}{Theorem}
\newtheorem{corollary}{Corollary}
\newtheorem{lemma}{Lemma}

\newtheorem{remark}{Remark}

\newtheorem{proposition}{Proposition}

\newcommand\myeq{\stackrel{\mathclap{\normalfont\mbox{def}}}{=}}
\usepackage[utf8]{inputenc} % allow utf-8 input
\usepackage[T1]{fontenc}    % use 8-bit T1 fonts
\usepackage{hyperref}       % hyperlinks
\usepackage{url}            % simple URL typesetting
\usepackage{booktabs}       % professional-quality tables
\usepackage{amsfonts}       % blackboard math symbols
\usepackage{nicefrac}       % compact symbols for 1/2, etc.
\usepackage{microtype}      % microtypography

\usepackage{mathtools}
\usepackage{xcolor}
\usepackage{color}
\usepackage{soul}
\title{Untangling tradeoffs between recurrence and self-attention in neural networks}

\author{%
  Giancarlo Kerg$^{1,2,}$ 
  \thanks{Indicates first authors. Ordering determined by coin flip. \newline
  \hangindent=4.3mm
  1: Mila - Quebec AI Institute, Canada \newline  
  2: Universit\'e de Montr\'eal, D\'epartement d'Informatique et Recherche Op\'erationelle, Montreal, Canada \newline
  3: Universit\'e de Montr\'eal, CIRRELT, Montreal, Canada\newline
  4: CIFAR senior fellow \newline
  5: Universit\'e de Montr\'eal, D\'epartement de Math\'ematiques et Statistiques, Montreal, Canada \newline
  \newline Correspondence to: <giancarlo.kerg@gmail.com>}
  \And
  Bhargav Kanuparthi $^{1,2,*}$
  \And 
  Anirudh Goyal $^{1,2}$
  \And
  Kyle Goyette $^{1,2,3}$
  \And
  Yoshua Bengio$^{1,2,4}$
  \And
  Guillaume Lajoie$^{1,2,5}$
}

\begin{document}

\maketitle

\begin{abstract}
% Blueprint 

Attention and self-attention mechanisms, are now central to state-of-the-art deep learning on sequential tasks.
However, most recent progress hinges on heuristic approaches with limited understanding of attention's role in model optimization and computation, and rely on considerable memory and computational resources that scale poorly. 
In this work, we present a formal analysis of how self-attention affects gradient propagation in recurrent networks, and prove that it mitigates the problem of vanishing gradients when trying to capture long-term dependencies by establishing concrete bounds for gradient norms.
Building on these results, we propose a relevancy screening mechanism, inspired by the cognitive process of memory consolidation, that allows for a scalable use of sparse self-attention with recurrence. 
While providing guarantees to avoid vanishing gradients, we use simple numerical experiments to demonstrate the tradeoffs in performance and computational resources by efficiently balancing attention and recurrence.
Based on our results, we propose a concrete direction of research to improve scalability of attentive networks.
\end{abstract}

%\vs{5}
\section{Introduction}
\label{introduction}

We live in a world where most of the information takes a sequential form, largely because it is delivered over time. Performing computations on streams of sequential inputs requires extracting relevant temporal dependencies and learning to recognize patterns across several timescales. 
Humans can effortlessly make associations relating events stored in memory which are far from each other in time and thus, capture long-term dependencies.

Historically, recurrent neural networks (RNNs) have been the deep network architecture of choice for this type of task since, just like neural circuits in the brain, they enable {\it dynamics} that can be shaped to interact with input streams. However, RNNs (including gated RNNs \cite{Schmidhuber:1997fq, gru}) still struggle with large timescales as their iterative nature leads to unstable information propagation \cite{Bengio:1994do, Pascanu:2013tw,Schmidhuber:1997fq,hochreiter1991}.This is because most standard RNNs rely on their current state $h_t$, a vector of fixed dimension, to represent a summary of relevant past information. Indeed, \citet{Bengio:1994do} showed that without making additional assumptions, storing information in a fixed-size state vector in a stable way necessarily leads to vanishing gradients when back-propagating through time (see also~\citep{hochreiter1991}).

% This is largely due to error gradients used for training, computed with back propagation through time (BPTT), that suffer from the exploding and vanishing problems~\cite{hochreiter1991,Bengio:1994do} which makes learning long-term dependencies very challenging. 
% % Most standard RNNs rely on a vector of fixed dimension $\textbf{h}_t$ to represent the state of the network at time $t$ as a summary of relevant past information. 
% % Learning in these networks happens using gradients of a loss function, e.g. obtained from back propagation through time (BPTT), where credit is assigned to each time step as error gradients flow backwards toward the past. RNNs are known to suffer from the exploding and vanishing gradient problem \cite{hochreiter1991,Bengio:1994do} which makes learning long-term dependencies using BPTT very challenging. 
% Moreover, most standard RNNs rely on a vector of fixed dimension $\textbf{h}_t$ to represent the state of the network at time $t$ as a summary of relevant past information.
% Combined with issues from BPTT
% Without making additional assumptions, this fact along with 
% this may even be intractable: \citet{Bengio:1994do} showed that storing information in a fixed-size state vector in a stable way necessarily leads to vanishing gradients.%, because even if somehow gradients can be propagated over long time steps, they may not be meaningful, this problem in literature has been called shattered gradients problems. \citep{balduzzi2017shattered}. 

%
Several attempts have been made to augment RNN dynamics with external memory to mitigate these issues \cite{end_to_end, NTM,RMC,graves2016hybrid}, but it is only recently that access to externally stored information has become effective with the introduction of {\it attention}, and more particularly {\it soft attention} mechanisms~\cite{attention}. 
Attention provides a way by which a system can dynamically access past states and inputs across several timescales, bypassing the need of sequential propagation and ignoring irrelevant information (or distractor information).
There is substantial empirical evidence that attention, especially {\it self-attention} (\citet{transformer, ke2018sparse}), is very helpful to improve learning and computations over long-term dependencies.
However, to the best of our knowledge, there is currently limited understanding of gradient scaling properties in the presence of attention. Moreover, attending over long sequences requires to hold inputs and/or past states in memory, a process that typically scales quadratically with sequence length.

Much like work from the '90s established formal results for gradient exploding/vanishing in deep/recurrent networks \cite{Bengio:1994do}, we believe it is crucial to establish similar theoretical tools for attention mechanisms, as these methods are under intense development where scalability and complexity are important issues.
In this paper, we contribute to this direction with a {\it formal analysis of gradient propagation in self-attentive systems which precisely quantify trade-offs between recurrence and attention}, offering valuable guarantees for attention mechanism development. Concretely exploiting these theorems, we propose a simple family of screening mechanisms to \textit{maximally reduce computational complexity and memory usage, while simultaneously maintaining good gradient propagation over large time scales}. 
Using simple tasks for their ease of interpretation, and their variety of computational demands, we illustrate the efficacy of this approach in numerical experiments.

The remainder of this paper is as follows. In Section \ref{fully_connected_graphs}, we give a brief outline of related cognitive processes and neural network mechanisms. In Section \ref{Theoretical analysis}, we present our central results: asymptotic guarantees for gradient propagation in self-attentive recurrent networks. To illustrate how to exploit these guarantees, in Section \ref{heuristics}, we showcase a simple \textit{relevancy screening mechanism} that aims to efficiently consolidate relevant memory, reducing the size of the computational graph from quadratic to linear in sequence length. Finally, in Section \ref{Experiments}, we compare various recurrent and attention models with our proposed relevancy screening mechanism on a series of simple numerical experiments, while, in Section \ref{analysis_section}, we analyze their gradient propagation properties together with their GPU usage.

\section{Background} %{Related cognitive processes and neural network mechanisms}
\label{fully_connected_graphs}
% ATTENTION
% Our brains seem to address segmenting streams of sensory input into meaningful representations of episodes and \emph{events}. Event segmentation allows functional representations that support temporal reasoning, an ability that arguably 
To perform complex tasks, our brains rely on mechanisms to encode and retrieve information to and from memory \citep{zacks2007event, radvansky2017event}. 
% Indeed, faced with a task, our brains can easily and {\it selectively} pluck context-relevant past information from memory, enabling both powerful multi-scale associations as well as flexible computations~\re{[CITE]}. 
% Both at the cognitive and at the physiological levels, there is evidence of context-dependent information "routing" mechanisms that enable this efficient propagation of information, although they are far from being understood~\re{[CITE]}. 
%
In contrast, standard RNNs follow rigid sequential dynamics as they are parametric i.e with a fixed-size state vector. Self-attention methods can overcome this limitation by giving access to previous past states for computing the next state. For the sake of the discussion, we call such RNNs, which are augmented by the memory of the past states as \emph{semi-parametric} RNNs. The use of soft-attention \citep{attention} in such models has improved performance on many tasks such as reading comprehension, abstractive summarization,
textual entailment and learning task-independent sentence representations \citep{parikh2016decomposable, lin2017structured, paulus2017deep, yang2019xlnet} as well as  in the self-supervised training of extremely large language models \citep{devlin2018bert, radford2019language}
due to their ability to handle long-term dependencies. 

%sparse compute
Intriguingly, the most notable advances in the use of attention is in purely attention-based systems such as the Transformer~\cite{transformer}, which completely foregoes recurrence and inspired some of the work listed above. While the performance of these systems is impressive, their memory and computation requirements grows quadratically with the total sequence length. 
%Scaling such models both in terms of memory as well as sequence length, has been an active area of research in deep learning. 
%
To address this issue, many variants that aim to "sparsify" the attention matrix have been proposed.
Notably, \citet{ke2018sparse} developed the Sparse Attentive Backtracking model (SAB), a self-attentive Long Short-Term Memory network (LSTM)~\cite{Schmidhuber:1997fq} that leverages sparsity by selecting only the top-$k$ states in memory based on an attention score, propagating gradients only to those chosen hidden states. Recently, \citet{zhao2019explicit} propose to use a similar top-$k$ attention, and \citet{child2019generating} introduce sparse masks which attends to roughly $\sqrt{n}$ locations in memory, implementing explicit selection methods for Transformers. Reformer models \citep{kitaev2020reformer} replace the dot-product attention
by locality-sensitive hashing, changing its complexity from $O(T^2)$ to $O(T)$, where $T$ is the sequence length. Finally, TransformerXL \citep{dai2019transformer} enables learning dependencies beyond a fixed length without disrupting temporal coherence and has resulted in state of the art performance in language models.

% \gc{Finally, TransformerXL \citep{dai2019transformer} decides which past activations to store in external memory, which enables learning dependencies beyond a fixed length without disrupting temporal coherence and has resulted in state of the art performance in language models.}

Still, most of these approaches naively sub-sample input streams for memory storage. Our brains on the other hand, seem to select relevant information from the recent past to commit to long term memory based on their relevancy, a process often referred to as memory consolidation~\cite{memory_reconsolidation}. Attempts at mimicking this sparse temporal selectivity process has shown great promise in a variety of contexts~\citep{NTM,Munkhdalai:2019vt,Harutyunyan:2019ws, goyal2019recurrent}, and our work aims to formalize this idea for self-attentive recurrent networks.

\section{Theoretical analysis of gradient propagation}
\label{Theoretical analysis}
%\vs{1}

In this section, we analyze the influence of self-attention onto gradient propagation in recurrent networks with self-attention. In order to do so let us first recall the equations of a recurrent neural network with self-attention. We note that even though we are using "vanilla RNNs" in the formulations of our results, any recurrent network can take its place (see Section \ref{Experiments} where we use LSTMs in the experiments). Let $x_t \in \mathbb{R}^m$ be the input and $h_t \in \mathbb{R}^n$ be the hidden state at time step $t$, satisfying the update equation for all $t\geq 1$, 
% In deep neural networks, an important quantity for learning that represents a direct proxy to information are error gradients. In this section, we formalize the influence of the cognitive biases outlined above, and that of event segmentation, onto gradient propagation in recurrent networks with self-attention.d
% Let us first recall the equations of a recurrent neural network with attention. 
%  \begin{tabularx}{\textwidth}{Xp{2cm}X}
%   \begin{equation}
%       h_{t+1} = \phi(V s_{t}+ U x_{t+1} + b)
%   \end{equation}
%   & &
%   \begin{equation}
%     s_t = f(h_t,c_t)
%   \end{equation}
%   \end{tabularx}
%\vs{1}
\begin{align}
    h_{t+1} &= \phi(V s_{t}+ U x_{t+1} + b)\\
    s_t &= f(h_t,c_t)
%\vs{2}
\end{align}
where $\phi$ is a non-linearity, $f:\mathbb{R}^n\times\mathbb{R}^n \rightarrow \mathbb{R}^n $, $V\in \mathbb{R}^{n \times n}$, $U \in \mathbb{R}^{n \times m}$, $b\in \mathbb{R}^n$ and 
% \vs{1}
% \begin{align}
    $c_t = \alpha_{1,t} h_1 + \alpha_{2,t} h_2 + \ldots + \alpha_{t,t} h_t$
% \vs{2}
% \end{align}
with $\alpha_{i,t} := \frac{\exp{(e_{i,t})}}{\sum_{j=1}^t \exp{(e_{j,t})}}$ and $e_{i,t} := a(s_{t-1},h_i)$, where 
$a: \mathbb{R}^{n}\times \mathbb{R}^n \rightarrow \mathbb{R}^n$ is the attention \textit{alignment function}. Throughout, we assume training is done via gradient descent of a cost function $L$ using back-propagation.

Oftentimes, one uses $s_t = f(h_t,c_t) = h_t + c_t$ (but concatenation would be more general), and for all $t>1$ and $1\leq j\leq t$, $a(s_{t-1},h_j) = v_a^T \cdot \tanh{(W_a \cdot s_{t-1}+ U_a \cdot h_j)}$, where $v_a \in \mathbb{R}^{n}$, and $W_a,U_a \in \mathbb{R}^{n \times n}$. The latter choice for alignment function is sometimes referred to as "additive self-attention" and was used in the original paper \cite{attention}. We emphasize that the results presented in this section hold independently of the choice of the alignment function as, we will discuss later in this section. Lastly, while results presented below are relatively succinct, their derivations are involved and we refer the interested reader to the Appendix for detailed proofs.

% Our goal in this section is to establish formal propagation rules for a system where multiple paths of signal propagation are possible. We would like to understand the relationship between skip connections (those coming from self-attention) and recurrent connections, as well as how the interplay between the two leads to good gradient propagation. In order to achieve this, we first decompose gradient terms with respect to different paths, and then evaluate these paths from the point of view of event segmentation. To do so, we first introduce notations and equivalences for gradients, and provide bounds for their magnitude under different sets of assumptions based on task statistics.
%and equivalences in gradient writing in form of Propositions 1 and 2 in section 3.1. Using this notation, we provide theorems for special cases. 

%\vs{1}
\subsection{Preliminaries}
%\vs{1}

Our goal in this section is to establish formal propagation rules for a system where multiple paths of signal propagation are possible. We would like to understand the relationship between skip connections (those coming from self-attention) and recurrent connections, as well as how the interplay between the two leads to good gradient propagation. In order to achieve this, we seek to analyze the asymptotic behaviour of $\|\nabla_{h_t} L\| = \|\left(\frac{d s_T}{ d h_t}\right)^T \nabla_{s_T} L\|$, as $T \rightarrow \infty$. We accomplish this by decomposing $\nabla_{h_t} L$ with respect to all possible gradient backpropagation paths, or in other words, by decomposing $\frac{d s_T}{ d h_t}$ into sums of products of Jacobian matrices corresponding to those gradient paths, using Proposition \ref{main_prop}.

\begin{proposition} \label{gradient_prop}
For all $t\geq 1$, $k\geq j\geq0$, $k'\geq 0$, let $E_{k'}^{(t)} = \frac{\partial s_{t+k'}}{\partial h_{t}}$, and $F_{k+1,j}^{(t)} = \frac{\partial s_{t+k+1}}{ \partial h_{t+j+1}} \cdot J_{t+j}+ 1_{j=k} \cdot \frac{\partial s_{t+k+1}}{\partial s_{t+k}}$, with $J_{t+j}$ the Jacobian matrix $\frac{d h_{t+j+1}}{ d s_{t+j}}$. Then, we have
%\vs{1}
\begin{align}
    \frac{d s_{t+k}}{d h_t} &= \sum_{s=0}^k \bar{\xi}_{0:k}^{(t)}(s)
%\vs{1}
\end{align}
where for all $s\geq 1$, 
% \begin{align}
$\bar{\xi}_{0:k}^{(t)}(s) = \sum_{0\leq i_1<\ldots<i_s<k}F_{k,i_s}^{(t)}\cdot F_{i_s,i_{s-1}}^{(t)}\cdot \ldots \cdot F_{i_2,i_1}^{(t)}\cdot E_{i_1}^{(t)}$
% \end{align}
and where $\bar{\xi}_{0:k}^{(t)}(0) = E_{k}^{(t)}$. 
% Here, for all $k'\geq 0$ and all $k\geq j$, we have 
% % \vs{1}
% % \begin{align}
% $E_{k'}^{(t)} &= \frac{\partial s_{t+k'}}{\partial h_{t}}$,  %\\
% $F_{k+1,j}^{(t)} &= \frac{\partial s_{t+k+1}}{ \partial h_{t+j+1}} \cdot J_{t+j}+ 1_{j=k} \cdot \frac{\partial s_{t+k+1}}{\partial s_{t+k}}$
% % \vs{1}
% % \end{align}
% where $J_{t+j}$ is the Jacobian matrix $\frac{d h_{t+j+1}}{ d s_{t+j}}$. 
(Proof in Appendix \ref{appendix_prelim}, Proposition \ref{main_prop})
\end{proposition}
Here, each term $F_{k,i_s}^{(t)}\cdot F_{i_s,i_{s-1}}^{(t)}\cdot \ldots \cdot F_{i_2,i_1}^{(t)}\cdot E_{i_1}^{(t)}$ corresponds to exactly one gradient path involving exactly $s+1$ skip connections going from $t$ to $t+k$, via the $s$ hidden states $h_{t+i_{s}+1}, \ldots, h_{t+i_{1}+1}$. In particular, $\bar{\xi}_{0:k}^{(t)}(s)$ is the sum of all terms containing exactly $s$ Jacobian matrices $J$, and thus the larger $s$ is, the more $\bar{\xi}_{0:k}^{(t)}(s)$ is prone to vanishing.  

{\bf Intuition:} In order to find paths that are not vanishing as $T \rightarrow \infty$, we want to find gradient paths with:
%\begin{itemize}
    %\item 
    {\bf (i)} a bounded path length $s$ so that the number of Jacobian matrices involved in the product is limited.
    %
    %\item 
   {\bf (ii)} attention scores that are sufficiently bounded away from $0$, so that the resulting product of attention scores is sufficiently bounded away from $0$ as well.
%\end{itemize}
%
In order to see how exactly the attention weights come into play via matrices $E$ and $F$, we refer to Proposition \ref{prop_D} from Appendix \ref{appendix_prelim}. %To present a formal treatment for this intuition in specific cases, we first need some notation.

{\bf Defintions:}  
Let us fix an integer $t\geq 1$, an integer $s \in \{1,2,\ldots,T-t\}$, and an ordered set of indices $i_1,i_2,\ldots, i_s \in \{0,1,\ldots,T-t-1\}$, verifying $i_1 \leq i_2 \leq \ldots \leq i_s$. 
%\vs{2}
\begin{itemize}
    \item For sequences $\{g(T)\}_{T\geq 1}$ and $\{f(T)\}_{T\geq 1}$, we say that \ul{$f(T)=\Omega(g(T))$} if there exists positive constants $c$ and $T_0$ such that $f(T)\geq c\cdot g(T)$ for all $T\geq T_0$.
    \item At time $t$, we call a past hidden state $h_i$ a \ul{\it relevant event} if the attention weight $\alpha_{i,t}$ is sufficiently bounded away from zero. 
    \item We call the $s$-tuple $(i_1,i_2,\ldots,i_s)$  a \ul{\it dependency chain $\gamma$ of depth $s$}, as it induces a gradient backpropagation path going via the $s$ hidden states $h_{t+i_{s}+1}, \ldots, h_{t+i_{1}+1}$.
    % \item if we further consider the attention weights $w_0,w_1,\ldots,w_s$ corresponding to the skip connections of $\gamma$, as $T\rightarrow \infty$, then $\gamma$ is said to have \textit{vanishing rate} $\Omega(w_0 \cdot w_1 \cdot \ldots \cdot w_s)$. 
    % \item we call \textit{dependency depth} the smallest depth among all dependency chains with vanishing rate $\Omega(1)$.
    \item We call \ul{\it dependency depth} the \textit{smallest} depth among all dependency chains where the product of the corresponding attention scores is $\Omega(1)$ as $T\to\infty$.
\end{itemize}
%
%
% \begin{proposition} \label{alpha_prop}
% With the notation from proposition \ref{prop_D} in appendix \ref{appendix_prelim},we can rewrite for all $k' \geq 0$ and all $k \geq j \geq 0$
% \vs{2}
% \begin{align}
%     E_{k'}^{(t)} &= \alpha_{t,t+k'}\cdot \tilde{D}_{k',0}^{(t)}+ 1_{k'=0} \tilde{R}_0^{(t)}\\
%     F_{k+1,j}^{(t)} &= \alpha_{t+j+1,t+k+1} \cdot D_{k+1,j}^{(t)} + 1_{k=j} \cdot R_{k+1}^{(t)}
% \vs{1}
% \end{align}
% \end{proposition}

% For the purpose of building this intuition, let us ignore for now the exact expressions of all $D, \tilde{D}, R$ and $\tilde{R}$ matrices involved, and note that each time we use a skip connection in order to construct a path from $T$ to $t$, we multiply the corresponding expression by a factor of the form $\alpha D$ (unless we pick consecutive indices in which case we use $\alpha D + R$). Thus intuitively, the path length corresponds exactly to the number of $D$ (and $R$) matrices, as well as the number of $\alpha$ scalars we are multiplying.

%\vs{1}
The central message  is that \textit{if the dependency depth is bounded from above and sufficiently small, then we mitigate gradient vanishing}. As we see below, task structure introduces different ways in which this may happen. We now present a formal treatment for specific cases, and lay the groundwork to take advantage of this structure during learning.

%--------
\subsection{Uniform relevance case}
Suppose each state has equal relevance in some task. What can be said about gradient propagation? This translates to having each attention weight $\alpha_{i,t} = 1/t$ for all $t \geq i\geq 1$. We then have dependency chains of depth $1$ but with vanishing rate $\Omega(1/T)$, as formalized in the following theorem (cf. \ref{uniform_attention_case})

\begin{theorem}\label{unif_thm}
Let $h_t$ be the hidden state at time $t$ of a vanilla RNN with uniform attention, under mild assumptions on the connectivity matrix $V$, and trained with respect to a loss $L$, then if $T$ is the total sequence length, we have 
\begin{align}\|\nabla_{h_t}L\| = \Omega(1/T)\end{align}
as $T\rightarrow \infty$. (proof in Appendix \ref{uniform_attention_case}, Theorem \ref{main_theorem_unif})
\end{theorem}
This corresponds to the case where all past events contribute equally error signals. We also note that this result is independent of the choice of the alignment function $a$ (cf. Remark \ref{unif_assumptions} in the Appendix \ref{uniform_attention_case}).

{\bf Intuition:} As a "worst case scenario" Theorem \ref{unif_thm} reveals the true trade-off of early self-attentive recurrent networks~\cite{attention}. On one hand, the lower bound obtained on gradient norm is substantially better than in a vanilla RNN without attention, where vanishing happens at an exponential rate, as opposed to a polynomial one here. This situation does not lend itself to sparse memory approaches as 
all events need to be held in memory, thus conserving quadratically scaling complexity.
In contrast many inputs and tasks do not call for uniform attention and naturally lend themselves to sparse dependency paths for computation. The next case treats this situation.
Nevertheless, this uniform attention bound is applicable in practice for two reasons: (1) typically, attention weights are initialized uniformly, and early training may result in gradients best described by this regime. (2) We experimentally verified that gradient propagation remains stable throughout training for a fully self-attentive RNN, where this bound is relevant, see Fig \ref{fig:grad_norm_plots} (Section \ref{analysis_section}).
%

%--------
\subsection{Sparse relevance case with bounded dependency depth}

Now let us look at a more realistic case where only a sparse subset of past states are relevant for the task at hand, and the gradient needs to access those states efficiently for good learning.
Figure~\ref{fig:heatmaps} illustrates this scenario by showing the attention scores for two input examples computed by a simple self-attentive model~\cite{attention}, trained on Copy and Denoise tasks respectively (see Section~\ref{Experiments}).
This structure introduces the possibility to impose sparsity in the computational graph, and to limit memory use. With these constraints in mind, the goal is to engineer dependency chains that enable best gradient propagation between these relevant events.
% \vs{2}
\begin{figure*}[t!]
    \centering
   \begin{subfigure}%[t]{0.35\textwidth}
   \centering
    \includegraphics[width=2.5in,  height=1.7in]{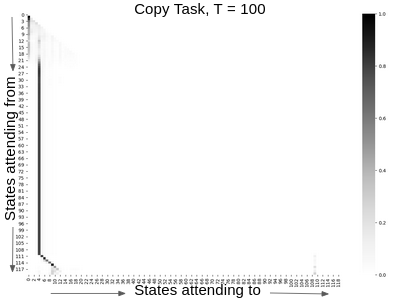}
    \end{subfigure}
    \begin{subfigure}%[t]{0.35\textwidth}
    \centering
    \includegraphics[width=2.5in,  height=1.7in]{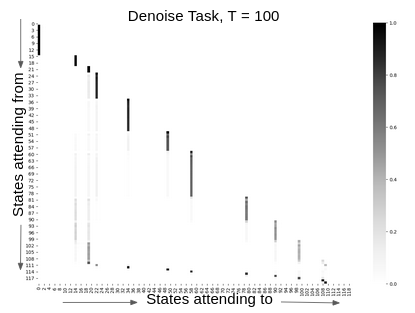}
    \end{subfigure}
    %\vs{1}
    \caption{\textit{ Magnitude of attention weights between states in a trained, fully recurrent and fully attentive model} (\citet{attention}). Each pixel in the lower triangle corresponds to the attention weight of the skip connection departing from the the state marked on the $y$-axis to the state marked on the $x$-axis. 
    %The states are numbered top to bottom on the $y$-axis, and left to right on the $x$-axis. 
    %the darker the pixel, the larger the corresponding attention weight. 
    Left shows Copy task, right shows Denoise task. Task details in Section~\ref{Experiments}}\label{fig:heatmaps}
    %\vs{4}
\end{figure*}

{\bf Notation}: We consider a $\kappa$-sparse attention mechanism of dependency depth $d$.
%\vs{2}
\begin{itemize}
    \item \ul{\it Sparsity coefficient}: $\kappa \geq 1$. Borrowing from the SAB model~\cite{ke2018sparse}, at each time step, attention is allowed at most $\kappa$ relevant events from the past. That is, for any $t$ there are at most $\kappa$ indices $i$ such that $\alpha_{i,t} \neq 0$, which gives rise to a sparse temporal segmentation via the most relevant events.
    % For simplicity, we suppose here that those past relevant events are all of equal importance. 
    % This is equivalent to having, for each time step $t\geq \kappa$, at most $\kappa$ indices $i \leq t$, such that $\alpha_{i,t} \neq 0$. For such an index $i$, we then have $\alpha_{i,t} \geq 1/\kappa$. 
    % This set of indices gives rise to a sparse temporal segmentation via the most relevant events. 
    % This corresponds exactly to the cognitive bias of event-based segmentation mentioned section \ref{fully_connected_graphs}. 
    
    \item \ul{\it Maximal dependency depth}: $d$. This is the maximal dependency depth across all time steps $t$.
    % We would like to have at least one dependency chain of depth less than $d$, with $d$ being fixed (and not scaling with $T$). This implies long skip connections relating events in the sequence, including up to far past (up to $T$ steps in the past). 
    %This corresponds exactly to the cognitive bias of having bounded dependency chains of relevant events mentioned in section \ref{fully_connected_graphs}.
\end{itemize}

\begin{theorem}\label{sparse_thm}
Let $h_t$ be the hidden state at time $t$ of a vanilla RNN with $\kappa$-sparse uniform attention mechanism of maximal dependency depth $d$, and under mild assumptions on the connectivity matrix $V$, then 
\begin{align}\|\nabla_{h_t}L\| = \Omega(1/\kappa^d)\end{align}
as $T\rightarrow \infty$. (proof in Appendix \ref{sparse_appendix}, Theorem \ref{main_theorem_sparse})
\end{theorem}
%\vs{1}
Similarly to uniform case, Theorem \ref{sparse_thm} is independent of the choice of the alignment function $a$ (cf. remark \ref{sparse_assumptions} in the appendix).

{\bf Intuition:}
Notice the dependency depth $d$ affects the lower bound exponentially, while $\kappa$ affects it polynomially. In other words, the number of relevant events attended to at each time step contributes far less to gradient vanishing than the number of events in the longest dependency chain. Theorem~\ref{sparse_thm} outlines the tradeoff between computational complexity as $T\to\infty$ and gradient propagation when balancing attention and recurrence. 
Attending directly to many relevant past events reduces $d$ and ensures good gradients at the expense of the complexity cost associated with storing past events and computing attention scores (the strategy employed by Transformers~\cite{transformer}). On the other hand, enforcing small sparsity coefficient $\kappa$ helps keep computational complexity low ($O(\kappa T)$), but forces the error gradient through recurrent paths, thereby augmenting the dependency depth $d$ and degrading gradient signal strength. 
Importantly, $\kappa$ and $d$ co-vary in ways that depend on the task's underlying relevancy structure, a point that is explained in detail in Appendix \ref{tradeoff_results} (See Fig \ref{fig:trade_off}). In the extreme case where $\kappa$ and $d$ are assumed to be bounded, we have $\Omega(1/\kappa^d)= \Omega(1)$, and thus we mitigate gradient vanishing. In other situations where $\kappa$ and $d$ scale in other ways, an explicit sparsification strategy can be derived by exploiting Theorem~\ref{sparse_thm}, as we illustrate in the next section.
%
% These heuristics were known before, but Theorem~\ref{sparse_thm} explicitly quantifies the distinct contributions to gradient propagation. Equipped with this result, we further develop these heuristics in what follows.

% This has been understood on an intuitive level before and led to heuristics that balance sparse attention and recurrence (e.g. ~\cite{ke2018sparse}), but Theorem~\ref{sparse_thm} is the first to quantify the distinct contributions to gradient propagation. Equipped with this result, we further develop these heuristics in what follows.
%\vs{2}
 \section{Relevancy screening mechanism}\label{heuristics}
 % Lets change the name of this section? Heuristic doesn't sounds good

Equipped with the results from the previous section, we wish to refine heuristics that strike a balance between good gradient propagation and computational/memory complexity. Building on the SAB model~\cite{ke2018sparse}, we remark that although sparse attention attends to the top-$\kappa$ events at any point in time, attention scores must be computed on all events stored in memory to extract the $\kappa$ best ones. Thus, the resource bottleneck is not controlled by $\kappa$, but rather by the number of stored events in memory. In SAB, there is a naive attempt to control this number by only recording network states at each 10 time steps. However, this reduces the size of the computational graph only by a constant factor, but retains $O(T^2)$ complexity. In contrast, Theorem~\ref{sparse_thm} tells us that the only important events to conserve for good gradient propagation are the {\it relevant} ones (also see Remark \ref{relevant_states_remark} in Appendix \ref{appendix_prelim}). Thus, we propose to reduce complexity while maintaining good gradient propagation by selectively storing events that are predicted to be relevant in the future, using a {\it relevancy screening mechanism}.

\begin{wrapfigure}{L}{0.46\textwidth}
    \begin{minipage}{0.46\textwidth}
    %\vs{8}
      \begin{algorithm}[H]
        \caption{Relevancy Screening}\label{relevancy_algo}
            \begin{algorithmic}[1]
            \State \textbf{procedure:} RelRNN$(\mathbf{s}_{t-1}, \mathbf{x}_t)$ 
            \newline \textbf{Require:} Previous macro-state - $\mathbf{s}_{t-1}$
            \newline \textbf{Require:} Input - $\mathbf{x}_{t}$, $\nu>0$, $\rho > 0$
            \newline \textbf{Require:} Short-term buffer $s_{t-1}^{(i)} \in S_{t-1}$
            \newline \textbf{Require:} Relevant set $r_{t-1}^{(i)} \in R_{t-1}$
            \State $h_t \leftarrow \phi(V\mathbf{s}_{t-1} + U\mathbf{x}_t + b)$
            \State $S_t = S_{t-1}$.add($h_t$)
            \If{$t - \nu > 0$}
            \State $S_t = S_t$.remove($h_{t-\nu}$) 
            \EndIf
            \If{$t - \rho > 0$ \textbf{and} $C(t-\rho) = True$}
            \State $R_t = R_{t-1}$.replaceWith($h_{t-\rho}$)
            \EndIf
            \State $M_{t} = [S_{t}, R_{t}]$
            \For{\textbf{all} $m^{(i)} \in M_{t}$}
            \State $\tilde{z}^{(i)} \leftarrow v_a^T \cdot \tanh{(W_a \mathbf{s}_{t-1} + U_a m^{(i)})}$
            \EndFor
            \State $z \leftarrow \textrm{softmax}(\tilde{z})$
            \State $\mathbf{s}_t = h_t + \sum_i z^{(i)}m^{(i)}$
            \State \textbf{return} $\mathbf{s}_t$
            \end{algorithmic}
      \end{algorithm}
      %\vs{10}
    \end{minipage}
\end{wrapfigure}

The idea is simple: devise a screening function $C(i)$ which estimates the future relevance of $h_i$, and store selected events in a {\it relevant set} $R_t = \{h_i | i<t \wedge  C(i)=True\}$ for future attention. In principle, one can explicitly control how $R_t$ grows with $t$, thus mitigating the complexity scaling outlined above.
Here, $C(i)$ could take many forms, the best of which depends on task structure. In what follows, we present an example screening mechanism meant to showcase the lessons learned from Theorem~\ref{sparse_thm}, but we refer the interested reader to Section \ref{Discussion} for further possibilities. 

We take inspiration from memory consolidation principles in human cognition~\cite{memory_reconsolidation}, which defines the transfer of events from short-term to long-term memory. We remark that for some tasks such as those depicted in Figure~\ref{fig:heatmaps}, relevance varies very little across time. 
%Conversely, most hidden states $h_t$ that didn't receive much attention in a short time window right after time step $t$, also seem to not receive much attention in the long term, thus not contributing very much to gradients (see Remark \ref{relevant_states_remark} in Appendix \ref{appendix_prelim}).
To implement relevancy screening for such tasks, at every time step $t$ we attend to two subsets of the past hidden states. We call the first subset a \textit{short-term buffer} $S_t = \{h_{t-\nu}, h_{t-\nu+1}, .. , h_{t-1}\}$ which consists of the hidden states of the last $\nu$ time steps, while the second subset is the relevant set $R_t$. We compute the {\it relevance score} at time step $i$, $\beta(i) = \sum_{j=i}^{i+\nu-1} \alpha_{i,j}$, measuring the integrated attention scores over our short-term buffer $S_t$. More precisely, $C(i)$ is satisfied if $\beta(i)$ is part of the top $\rho$ relevance scores when compared to all previously observed hidden states, where $\rho$ is a fixed hyper-parameter satisfying $\rho \geq |R_t|$ for all $t$. The pseudo-code in Algorithm \ref{relevancy_algo} describes the screening mechanisms and the interaction between the short-term buffer $S_t$ and a finite size relevant set $R_t$.  '.replaceWith()' is a function replacing the hidden state with the lowest relevance score by the hidden state in the argument.

To see how the relevancy screening mechanism is grounded in the theory developed in Section \ref{Theoretical analysis}, note that the sets $S_t$ and $R_t$ give rise to a sparse attention mechanism with sparsity coefficient $\kappa$ satisfying $\kappa = \nu + \rho  \geq |S_t|+|R_t|$. Hence, memory complexity is constant while the $O(T^2)$ bottleneck of computational complexity is replaced by $O((\rho+\nu)\cdot T) = O(T)$. Lastly, applying Theorem~\ref{sparse_thm}, we get the following guarantee for all $t\geq 0$: $\|\nabla_{h_t}L\| = \Omega(1/(\rho+\nu)^d)$ as $T\rightarrow \infty$. Thus the choices of $\nu$ and $\rho$ not only directly impact computational complexity and gradient propagation, but also indirectly influence gradient propagation via the implicit effect of $\kappa = \nu + \rho$ on $d$ as already discussed in Section \ref{Theoretical analysis}. Finally, as already mentioned, see Fig \ref{fig:trade_off} in Appendix \ref{tradeoff_results}, where we perform an experimental trade-off analysis between $\kappa$ and $d$ by tweaking $\rho$ and $\nu$ in the relevancy screening mechanism.

\section{Experiments}
\label{Experiments}
%\vs{1}

% The goal of this section is to empirically illustrate the advantages afforded by selective sparsity in attentive recurrent networks, based on the \textit{Relevancy Screening Mechanism} proposed in Section~\ref{heuristics}. 
Before describing experiments, we make a few remarks. 
First, we stress that Relevancy Screening can be applied to any semi-parametric attentive model but we refer to the version presented below, which uses an RNN/LSTM base, as  RelRNN/RelLSTM ("{\it Relevance RNN /LSTM}"). 
Second, our objective is not to find state-of-the-art performance but to highlight the advantages of event relevancy and selective sparsity.   
Finally, we note that relevancy-based sparsity does not necessarily improve performance over fully attentive models, but rather allows efficient and scalable usage. As we show below, RelRNN and RelLSTM perform almost identically to other self-attentive recurrent models (e.g.~\cite{attention,ke2018sparse}) on simple tasks, but use considerably less memory and compute complexity. In what follows, we denote MemRNN/MemLSTM for vanilla self-attention RNN/LSTM as defined in \cite{attention}.
%
% We consider two task categories to highlight distinct impacts of selective attention and recurrence. The first category specifies tasks with sparse dependency chains,  and the second one those with dense temporal dependencies. All implementation details and hyper-parameters can be found in the Appendix \ref{extra_results}. 
We also refer to Appendices \ref{tradeoff_results}, \ref{minigrid}, \ref{extra_results} for additional experimental results and implementation details.

% : in fact the size of the computational graph for RelRNN is $O((\rho+\nu)\cdot T)$ ($\nu$ being the size of the short-term buffer, and $\rho$ the size of relevant set), while that for fully attentive models is $O(T^2)$. The latter often fails to train because of overflow issues, as we outline below.

% verify the validity of our analytical findings that 

% illustrate 

% by looking at several well known benchmark tasks and comparing its performance with other standard baselines. Note that the primary purpose of this section is not to claim state of the art performance on these tasks but to highlight the advantages of leveraging the Theorems derived in Section \ref{Theoretical analysis}. We use the name RelRNN ({\it Relevance RNN}) to refer to our proposed model, which, unless specified otherwise, performs almost identically to fully attentive models from \citet{attention}, while using considerably less memory and compute complexity: in fact the size of the computational graph for RelRNN is $O((\kappa+\nu)\cdot T)$ ($\nu$ being the size of the short-term buffer, and $\kappa$ the size of relevant set), while that for fully attentive models is $O(T^2)$. Hence, as specified below, fully attentive models fail to train because of overflow issues in some instances, whereas RelRNN has no problems.

%\vs{1}
%\vs{1}
\subsection{Tasks with sparse dependency chains}

A good stereotypical task type that captures sparse sequences of important events are memorization tasks. Here, the network has to memorize a sequence of relevant characters presented among several non-relevant ones, store it for some given time delay and output them in the same order as they were read towards the end of the sequence. 

{\bf Copy task}~\citep{hochreiter1997long}:  The characters to be copied are presented in the first $10$ time steps, then must be outputted after a long delay of $T$ time steps (see full description in~\citet{uRNN}). Thus, all the \textit{relevant events} occur in the first $10$ time steps. This can be corroborated by the attention score found in Figure \ref{fig:heatmaps} which was generated using full self-attention. \citet{henaff2016recurrent} show that orthogonal RNNs (orth-RNN) provide an optimal solution. We also consider expRNN \citep{DBLP:journals/corr/abs-1901-08428} which does optimization in the unitary space and is so far the best purely performing recurrent model for large time delays for this task. 
%The network size for all models was fixed to $128$. %AG: this is also a bit complicated in general.
% RMSprop \citep{tieleman2012lecture} with a learning rate of $0.0002$ and standard cross-entropy loss was used for all experiments. Each model was trained for $100000$ iterations. For RelRNN we set the hyperparameter $\kappa = 10$ and $\nu=10$.
%

\begin{table}
  \parbox{.50\linewidth}{
    %\centering
    \small
    \caption{Results for Transfer Copy task. }
    \begin{tabular}{cc c c c c}
    \toprule
    $T$ & $100$ & $200$ & $400$ & $2000$ & $5000$ \\
    \midrule
    orth-RNN & $99\%$ & $4\%$ & $16\%$ & $10\%$ & $0\%$ \\
    expRNN & $100\%$ & $86\%$ & $73\%$ & $58\%$ & $50\%$ \\
    MemRNN & $99\%$ & \bm{$99\%$} & \bm{$99\%$} & $92\%$ & OOM \\
    RelRNN & \bm{$100\%$} & \bm{$99\%$} & \bm{$99\%$} & \bm{$99\%$} & \bm{$99\%$} \\
    \midrule
    LSTM & $99\%$ & $64\%$ & $48\%$ & $19\%$ & $14\%$ \\
    h-detach & $100\%$ & $91\%$ & $77\%$ & $51\%$ & $42\%$ \\
    SAB & $99\%$ & $95\%$ & $95\%$ & $95\%$ & $95\%$ \\
    RelLSTM & \bm{$100\%$} & \bm{$99\%$} & \bm{$99\%$} & \bm{$99\%$} & \bm{$99\%$} \\
    \bottomrule
    \end{tabular}

    %\caption{Results for the PTB and MNIST tasks.} OOM stands for Out of Memory.
    \label{tab:transfer_copy}
    %\vs{2}
  }
 \hfill  
 \parbox{.50\linewidth}{
    %\centering
    \small
    \caption{Results for Denoise task.}
    \begin{tabular}{cccccc}
        \toprule
         $T$ & $100$ & $300$ & $500$ & $1000$ & $2000$ \\
         \midrule
         orth-RNN & $90\%$ & $71\%$ & $61\%$ & $29\%$ & $3\%$ \\
         expRNN & $34\%$ & $25\%$ & $20\%$ & $16\%$ & $11\%$ \\
         MemRNN & $99\%$ & \bm{$99\%$} & \bm{$99\%$} & \bm{$99\%$} & OOM \\
         RelRNN & \bm{$100\%$} & \bm{$99\%$} & \bm{$99\%$} & \bm{$99\%$} & \bm{$99\%$}\\
         \midrule
         LSTM & $82\%$ & $41\%$ & $33\%$ & $21\%$ & $15\%$ \\
         GORU & $92\%$ & $93\%$ & $91\%$ & $93\%$ & $73\%$ \\
         SAB & $99\%$ & \bm{$99\%$} & \bm{$99\%$} & \bm{$99\%$} & \bm{$99\%$} \\
         RelLSTM & \bm{$100\%$} & \bm{$99\%$} & \bm{$99\%$} & \bm{$99\%$} & \bm{$99\%$}\\
         \bottomrule
    \end{tabular}
    %\caption{Results for Denoise task}
    \label{tab:denoise_table}
    %\vs{1}
  } 
  \vskip -0.2in
\end{table}

Table~\ref{tab:copy_table} (Appendix \ref{extra_results}) reports test performances of orth-RNN, expRNN, MemRNN, SAB, RelRNN and RelLSTM for $T = \{100, 200, 300, 500, 1000, 2000 \}$ on the Copy Task. We find that orth-RNN solves this task up to $T=500$, but that accuracy decays beyond that point, similarly to LSTM. RelRNN, RelLSTM, SAB and expRNN perfectly solve this task with \bm{$100\%$} accuracy for all $T$, while Fig \ref{fig:exp} in Appendix \ref{extra_results} shows that RelRNN learn copy and denoise tasks with significantly fewer number of updates as compared to other baselines. MemRNN solves this task until $T=100$ but overflows memory (OOM) afterwards.

{\bf Transfer Copy task}: An important advantage of sparse attentive recurrent models such as RelRNN is that of generalization. This is illustrated by the Transfer Copy scores~\citep{hochreiter1997long} where models are trained on Copy task for $T=100$ and evaluated for $T > 100$.
Table~\ref{tab:transfer_copy} shows results for the models listed above, in addition to h-detach~\citep{arpit2018h}, an LSTM-based model with improved gradient propagation.
Importantly, where purely recurrent networks performed well on the original task, all fail to transfer, with discrepancy growing with $T$. As expected, MemRNN and SAB keep good performance but RelRNN outperforms them, with almost perfect performance for all $T$. While both SAB and RelRNN use sparse memory storage and retrieval, the distinguishing factor is RelRNN's use of relevancy screening, indicating it's importance for transfer.  The performance of RelLSTM on Transfer Copy is exactly the same as RelRNN.

{\bf Denoise task}~\citet{jing2019gated}: This generalizes the Copy task as the symbols that need to be copied are now randomly distributed among the $T$ time steps, requiring the model to selectively pick the inputs that need to be copied.
We test our method against all the previously mentioned models in addition to GORU \cite{jing2019gated} for various values of $T$ (Table~\ref{tab:denoise_table}). RelLSTM performs exactly as RelRNN and again, we see RelRNN maintain complete performance across all $T$ values, outperforming all purely recurrent models. MemRNN performs as RelRNN/RelLSTM but fails to train due to memory overflow beyong $T=500$.

\subsection{Tasks with dense temporal dependencies}

In contrast to sparse information found in the tasks above, we now illustrate RelRNN and RelLSTM's performance on tasks with densely distributed information on long sequences. 

%\vs{14}
\begin{wraptable}{L}{0.45\textwidth}
    %\vs{7}
    \small
    \caption{PTB and pMNIST results.}
    %\vs{-2}
    \begin{tabular}{c ccc}
        \toprule
         & \multicolumn{2}{c}{\textbf{PTB Task}} & \textbf{pMNIST} \\
         Model & BPC & Accuracy & Accuracy \\
         \midrule
         RNN & $1.56$ & $66\%$ & $90.4\%$ \\
         orth-RNN & $1.53$ & $66\%$ & $93.4\%$ \\
         expRNN & $1.49$ & $68\%$ & \bm{$96.6\%$}\\
         RelRNN & \bm{$1.43$} & \bm{$69\%$} & $92.8\%$ \\
         \midrule
         LSTM & \bm{$1.36$} & \bm{$73\%$} & $91.1\%$ \\
         h-detach & - & - & $92.3\%$ \\
         SAB & $1.37$ & - & $94.2\%$  \\
         RelLSTM & \bm{$1.36$} & \bm{$73\%$} & \bm{$94.3\%$}\\
         \bottomrule
    \end{tabular}
    %\vs{4}
    %\caption{Results for Transfer Copy task.}
    \label{tab:ptb_res}
\end{wraptable}

Here, we perform tests on pMNIST  \cite{le2015simple}, a variant of MNIST~\cite{mnist} where pixels are fed sequentially in a permuted order to the network, as well as character level Penn Tree Bank corpus (PTB)~\citep{marcus1993building} where the next letter in a text needs to be predicted. 

\iffalse
\begin{wraptable}{L}{0.45 \textwidth}
\centering
\small
\caption{Average Train and Test Rewards for MiniGrid Reinforcement Learning task. The models were trained on the smaller version of the environment and tested on the larger version to test to generalization of the solution learned.}
    \label{tab:rl_results}
    %\vs{-2}
\begin{tabular}{cccc}
        \toprule
         Environment & MemLSTM & RelLSTM \\
         \midrule
         & \multicolumn{2}{c}{\textbf{Train}} \\
         \midrule
         RedBlueDoors-6x6  & \bm{$0.97$} & \bm{$0.97$} \\
         GoToObject-6x6 & \bm{$0.85$} & $0.84$ \\
         MemoryS7 & $0.4$ & $0.94$  \\
         GoToDoor-5x5 & $0.17$ & $0.25$ \\
         Fetch-5x5 & $0.42$ & \bm{$0.5$} \\
         DoorKey-5x5 & \bm{$0.94$} & $0.93$ \\
         \midrule
         & \multicolumn{2}{c}{\textbf{Test}} \\
         \midrule
         RedBlueDoors-8x8   & \bm{$0.95$} & \bm{$0.95$} \\
         GoToObject-8x8 &  $0.66$ & \bm{$0.74$} \\
         MemoryS13 &  $0.24$ & \bm{$0.30$}  \\
         GoToDoor-8x8 &  $0.11$ & \bm{$0.15$} \\
         Fetch-8x8 &  $0.44$ & \bm{$0.45$} \\
         DoorKey-16x16 &  $0.31$ & \bm{$0.44$} \\
         \bottomrule
    \end{tabular}
\end{wraptable} 
\fi

\vspace*{-4 pt}
See Table~\ref{tab:ptb_res} for results. Implementation details and further test data found in Appendix \ref{extra_results}, including attention heatmaps such as the ones found in Figure~\ref{fig:heatmaps}, showing dense attention for RelRNN in both tasks.
We note that gated RNNs such as LSTMs are known to perform well here, and that orthogonal RNNs such as those tested here are also very good.
The full attention model (MemRNN) fails to train on the optimization setup used here for both tasks, again due to overflow in memory.

\section{Analysis} \label{analysis_section}
%%%%%%
In this section we analyze the maximal GPU usage and gradient norm of $\|\nabla_{h_t}L\|$ across time $t$ for the Denoise Task. All the models were run using a NVIDIA TitanXP GPU and their peak usage was recorded in order to quantify the amount of computational resources used for each of them. We varied sequence length $T$ from $200$ to $2000$ in order to measure the trend in the usage. 
%
% For gradient norms we considered the Denoise Task with $T=1000$. Each model was first updated $20000$ times. Then we calculated $\|\nabla_{h_t}L\|$ in log scale, and averaged it over $10$ mini-batches 
To measure propagating gradients as a function of $t$, we trained models on $T=1000$ and computed $\log\|\nabla_{h_t}L\|$.%We report the averages over 10 mini-batches.

As illustrated in Figure \ref{fig:grad_norm_plots} (center), we confirm MemRNN scales quadratically with $T$, same as SAB which shows an improvement but only by a constant factor. We also confirm that RelLSTM scales linearly with $T$ similar to RNN and LSTM.
Figure \ref{fig:grad_norm_plots} (left) shows that the gradient norms for RNN explode and for LSTM vanish as $t$ increases. The gradient norms of all attention models were stable, as expected from the results of Section \ref{Theoretical analysis}.
To better visualize the interplay between gradient norm and GPU usage, Figure \ref{fig:grad_norm_plots} (right) shows the final averaged log gradient norm against Max GPU usage for different times $T=\{400, 600, 800\}$. 
As expected, purely recurrent models (RNN, LSTM) show very little GPU usage differences across distinct $T$ values, while their performance and gradients degrade with increasing $t$. Note that the RNN's gradients explode while the LSTM's vanish, both exponentially in $t$. Standard self attentive models (MemRNN, SAB) on the other hand, show opposite trends, with stable gradients but GPU usage quadratically increasing in $T$. As expected from Theorem 2 (Section \ref{Theoretical analysis}), RelLSTM shows both stable gradients and stable GPU usage\footnote{The measurements for both GPU usage and gradient norm are identical for both RelLSTM and RelRNN.}. 
% we can see that RelLSTM barely moves and is situated all the way to the left close to 0 on the y-axis. We can also see that models like MemRNN and SAB drift off to the right, while RNN / LSTM drift off to the top / bottom lines respectively. 

The optimal trade-off between memory usage and good gradient propagation achieved by RelLSTM highlights the importance of a dynamic memory that attempts to predict relevancy in order to only store exactly those events that help with learning. We note the Denoise task has a small number of relevant events and that not all tasks share this structure. Nevertheless, this experiment highlights how important resource gains can be made by shifting efforts from offsetting memory growth by a constant factor, to a relevancy screening method.

  %can see that RNN gradient is exploding, while LSTM gradient is vanishing. All three models SAB, MemRNN and RelRNN have good gradient propagation \gc{with slight differences in stability: SAB seems more stable than RelRNN. Are we adding RelLSTM?}. Meanwhile, Figure ~\ref{fig:gpu_usage} shows Max GPU usage for all models as a function of total sequence length. We can see how MemRNN and SAB both scale quadratically, with SAB having a constant delay because it attends to only every 10 time steps. RelRNN, RNN and LSTM all scale linearly. 
\begin{figure*}[t!]
    \centering
    \begin{subfigure}%[t]{0.35\textwidth}
   \centering
    \includegraphics[scale=0.2]{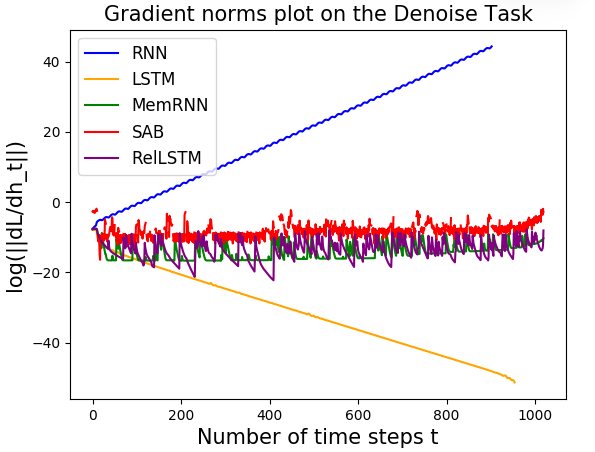} 
    \end{subfigure}
    \begin{subfigure}%[t]{0.35\textwidth}
    \centering
    \includegraphics[scale=0.2]{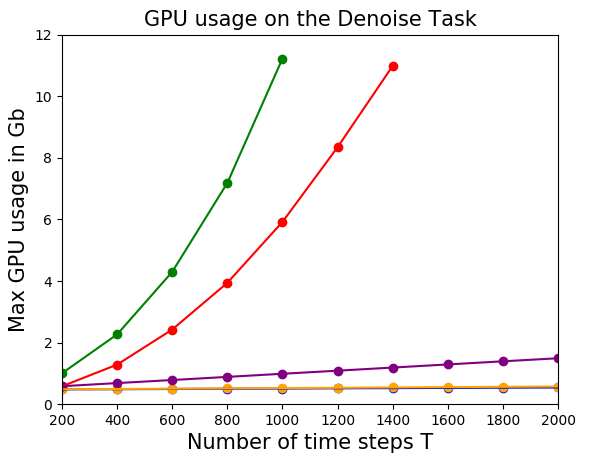}
    %\label{fig:gpu_usage}
    \end{subfigure}
    \begin{subfigure}%[t]{0.35\textwidth}
    \centering
    \includegraphics[scale=0.25]{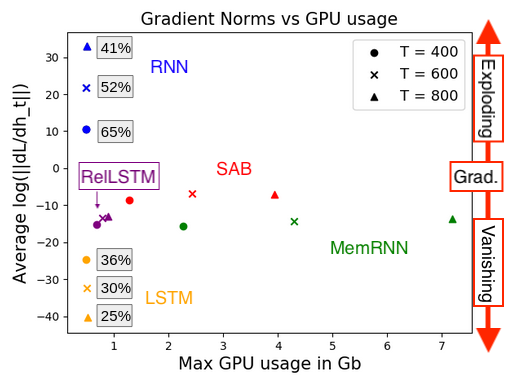}
    %\label{fig:scatter_plot}
    \end{subfigure}
    
    \caption{\textbf{(Left)} gradient norm plots of $\|\nabla_{h_t}L\|$ in log scale after training for Denoise Task with $t$ ranging from 0 (latest time step) to 1000 (furthest time step).\textbf{(Center)} Maximal GPU usage as a function of total sequence length $T$.\textbf{(Right)} Mean log gradient norm v.s. Max GPU usage for $T=400,600,800$. Model testing accuracy is $100\%$ unless indicated by marker label (see Table~\ref{tab:denoise_table}).}
    \label{fig:grad_norm_plots}
\end{figure*}

% \begin{wrapfigure}{r}{0.5\textwidth}
%   \begin{center}
%     \includegraphics[scale=0.22]{figures/scatter_plot.png}
%     \label{fig:scatter_plot}
%   \end{center}
%   \caption{shows the trade-off between gradient propagation and compute/memory usage.}
% \end{wrapfigure}

% \subsection{RelLSTM: Relevancy screening for LSTM} \label{lstm_experiments}

% We tested our Relvenacy Screening Mechanism on RNNs because they are the basis of the theoretical results proved in Section \ref{Theoretical analysis}. Although the Theorems are based on non-gated RNNs, the results can be extended to gated RNNs in a natural way. The state-of-the art for semi-parametric memory models used gated networks so we also tested our Relevancy Screening Mechanism as proposed in Section \ref{heuristics} on LSTMs (RelLSTM). We find that vanilla LSTMs struggle to solve the Copy and the Denoise task for large $T$, but RelLSTM is easily able to achieve an accuracy of $100\%$ for both the previously mentioned tasks and the Transfer copy for task for all $T$. For the PTB, task we find that an LSTM (1 layer) and RelLSTM (1 layer) both give a test bpc of $1.39$ and an accuracy of $71\%$. For plots, see Appendix \ref{extra_results}.

%\vs{3}
\section{Conclusion \& Discussion} 
\label{Discussion}
%\vs{1}

% Inspired by human cognition, we leverage inductive biases for temporal reasoning and credit assignment in recurrent networks with attention. 

Our main contribution is a formal analysis of gradient propagation in self-attention RNNs, from which we derive two quantities that are governing gradient propagation: sparsity and dependency depth. Meanwhile we identify event relevancy as a key concept to efficiently scale attentive systems to very long sequential computations. This is illustrated via a {\it Relevancy Screening Mechanism}, inspired by the cognitive process of memory consolidation in the brain, that efficiently selects network states, called relevant events, to be committed to long-term memory based on a screening heuristic operating on a fixed-size short-term memory buffer. We showcase the benefits of this mechanism in an attentive RNN and LSTM which we call RelRNN and RelLSTM respectively, using simple but illustrative numerical experiments, and demonstrate the optimal trade-off between memory usage and good gradient propagation it achieves. 

As outlined in Sections \ref{Theoretical analysis} and \ref{heuristics}, this trade-off is a reflection of the task-specific balance between sparsity and dependency depth parameters. 
%
% To exploit this balance in different cases, we acknowledge that other relevancy screening mechanisms could be developed. 
While our proposed relevancy screening mechanism exploits "local" attention scores (measured while events are in short-term memory buffer), we acknowledge other types of relevancy could be formulated with heuristics better suited to distinct environments. For instance, promising directions include leveraging predictive coding principles to select "surprising events", or neural networks could be used to learn the screening function $C(i)$ in an end-to-end fashion.

\section*{Broader Impact}

We provide a framework for researchers to shape gradient propagation and memory footprint in self-attentive RNNs, which is helpful in tasks requiring ongoing online predictions that cannot be based on future inputs (i.e. in an online sequential setting) and where long-term credit assignment is crucial, such as various RL tasks \cite{Harutyunyan:2019ws, TVT}. The added resource gains can save GPU hours and thus have a positive environmental impact. Along this line, we firmly believe that researchers should take environmental impact of model training seriously, and we are hopeful that our work contributes to this direction.

Meanwhile, the theoretical tools provided in the proofs lay the ground for more theoretical work on attentive systems to emerge in the future. 
More effective RNN models can amplify already existing biases in RNN-based NLP systems through an increased exposure to bias. Finally, we cannot exclude that the cognitive inductive bias we use to build our relevancy screening mechanism may induce prediction quality disparity (e.g. in language modelling) because of the memory tokens it throws away. 

\begin{ack}
We would like to thank Gauthier Gidel, Naz Sepah, Yassine Yaakoubi, Sarthak Mittal, and Chen Sun for useful discussions. BK acknowledges IVADO Masters Excellence Scholarship.
YB acknowledges support from CIFAR, Microsoft and NSERC. 
GL is funded by an NSERC Discovery Grant (RGPIN-2018-04821), an FRQNT Young Investigator Startup Program (2019-NC-253251), and an FRQS Research Scholar Award, Junior 1 (LAJGU0401-253188).
\end{ack}

\bibliography{GC_bib.bib}
\bibliographystyle{plainnat}
\appendix 

\newpage
\begin{center}
{\bf Supplemental Material for:}\\
Untangling tradeoffs between recurrence and self-attention in artificial neural networks
\end{center}

\setcounter{theorem}{0}
\setcounter{lemma}{0}
\setcounter{proposition}{0}
\setcounter{corollary}{0}

\section{Theoretical analysis of gradient propagation}
\label{math}
\subsection{Notational convention}
In this paper, we use the notation $\frac{df}{dx}$ to denote the total derivative of $f$ with respect to $x$, and $\frac{\partial f}{\partial x}$ to denote the partial derivative of $f$ with respect to $x$. \\

If we assume $f:\mathbb{R}^n \rightarrow \mathbb{R}^m$, and $x\in \mathbb{R}^n$, then 
$\frac{df}{dx}$ denotes the Jacobian matrix $J_f$ such that \begin{align}(J_f)_{ij}= \frac{df_i}{dx_j}\end{align}
In particular, with this notation, we have that if a function $L: \mathbb{R}^{m}\rightarrow \mathbb{R}$, and $y\in \mathbb{R}^m$ then
$\frac{dL}{dy}$ is a row vector, while the conventional notation for $\nabla_y L$ indicates a column vector. In other words, $(\nabla_y L)^T = \frac{d L}{d y} $. Hence if $L$ is a function of $f(x)$, then  
\begin{align}\frac{dL}{dx} = \frac{dL}{df}\cdot \frac{df}{dx}\end{align}
while 
\begin{align}\nabla_x L = \left(\frac{df}{dx}\right)^T\cdot \nabla_{f(x)} L = J_f^T \cdot \nabla_{f(x)} L\end{align}
Similarly, we have that $\frac{\partial L}{\partial y}$ is a row vector.

\noindent\rule{\textwidth}{1pt} 
\subsection{Preliminary results} \label{appendix_prelim}
Let \begin{align}s_t=\psi_t(h_1,h_2,\ldots,h_t,s_{t-1})\end{align} where \begin{align}h_{i+1} = \phi(V s_i + U x_{i+1}+b)\end{align}
\begin{lemma}
For all $t,k \geq 0$, we have 
\begin{align}\frac{d s_{t+k+1}}{d h_t} = \frac{\partial s_{t+k+1}}{\partial h_t} + \left(\sum_{j=0}^{k} \frac{\partial s_{t+k+1}}{\partial h_{t+j+1}}\frac{d h_{t+j+1}}{d h_t} \right)+ \frac{\partial s_{t+k+1}}{\partial s_{t+k}}\frac{d s_{t+k}}{d h_t}\end{align}
\end{lemma}
\begin{proof} Follows directly from the following multivariable chain rule: if \begin{align}g(t)=f(g_1(t),g_2(t),\ldots,g_n(t))\end{align}
then 
\begin{align}\frac{dg}{dt} = \sum_{i=1}^n \frac{\partial f}{ \partial g_i} \frac{d g_i}{dt}\end{align}

\end{proof}
\noindent\rule{\textwidth}{1pt}
\begin{lemma}\label{lemma2}
If we further denote the Jacobian matrix $J_k = \frac{\partial s_{k+1}}{\partial h_k}$, then we get that for all $t,k \geq 0$, we have 
\begin{align}\frac{d s_{t+k+1}}{d h_t} = \frac{\partial s_{t+k+1}}{\partial h_t} + \sum_{j=0}^{k} \left(\frac{\partial s_{t+k+1}}{\partial h_{t+j+1}}\cdot J_{t+j}+ 1_{j=k}\cdot \frac{\partial s_{t+k+1}}{\partial s_{t+k}}\right)\cdot \frac{d s_{t+j}}{d h_t}\end{align}
\end{lemma}

\begin{proof}
Follows directly from the observation that
\begin{align}\frac{d h_{t+j+1}}{d h_t} = \frac{\partial h_{t+j+1}}{\partial s_{t+j}} \frac{d s_{t+j}}{d h_t} = J_{t+j} \cdot \frac{d s_{t+j}}{d h_t}\end{align}
\end{proof}
\noindent\rule{\textwidth}{1pt}
\begin{remark} \label{remark1}
Let us denote \begin{align}C_{k+1}^{(t)}= \frac{d s_{t+k+1}}{d h_t}\end{align} \begin{align}E_{k+1}^{(t)} = \frac{\partial s_{t+k+1}}{\partial h_{t}}\end{align} and \begin{align}F_{k+1,j}^{(t)} = \frac{\partial s_{t+k+1}}{ \partial h_{t+j+1}} \cdot J_{t+j}+ 1_{j=k} \cdot \frac{\partial s_{t+k+1}}{\partial s_{t+k}}\end{align}
and thus the recursion formula in Lemma \ref{lemma2} rewrites as 
\begin{align}C_{k+1}^{(t)} = E_{k+1}^{(t)}+ \sum_{j=0}^k F_{k+1,j}^{(t)}\cdot C_j^{(t)}\end{align}
The next two results highlight how to solve this recursion.\\
\end{remark}
\noindent\rule{\textwidth}{1pt}
\begin{lemma} Let $C_i,E_i, F_{i,j} \in \mathbb{R}^{n \times n}$ such that for all $k\geq 0$, we have 
\begin{align}C_{k+1} = E_{k+1} + \sum_{j=0}^k F_{k+1,j}\cdot C_{j}\end{align}
Then for all $k\geq 1$, we have 
\begin{align}\boxed{C_k = \xi_{0:k} C_0 + \sum_{r=1}^k \xi_{r:k} E_r}\end{align}
where \begin{align}\xi_{r:k} = \sum_{s=1}^{k-r} \xi_{r:k}(s)\end{align}

with \begin{align}\xi_{r:k}(s) = \sum_{r=i_1<\ldots<i_{s+1}=k}F_{i_{s+1},i_{s}}\cdot F_{i_{s-1},i_{s-2}}\cdot \ldots \cdot F_{i_2,i_1}\end{align} and $\xi_{k:k} = \textrm{Id}$.
\end{lemma}

\begin{proof} Let us prove the statement by induction on $k \geq 1$.\\
For $k=1$, we have \begin{align}C_1 = E_1 + F_{1,0} C_0 = \xi_{1:1} E_1 + \xi_{0:1} C_0\end{align}
Now let us assume the statement to be true for $k$, then we get
\begin{align}
    C_{k+1} &= E_{k+1} + \sum_{j=0}^k F_{k+1,j}\cdot \left(\xi_{0:j}C_0+\sum_{r=1}^j \xi_{r:j}E_r\right)\\
    &= E_{k+1}+ \left(\sum_{j=0}^k F_{k+1,j}\cdot  \xi_{0:j}\right)\cdot C_0+\sum_{j=0}^k \sum_{r=1}^j F_{k+1,j}\xi_{r:j}E_r\\
    &= E_{k+1} + \xi_{0:k+1} C_0+\sum_{r=1}^k \left(\sum_{j=r}^k F_{k+1,j}\xi_{r:j}\right)\cdot E_r \\
    &= \xi_{k+1:k+1}E_{k+1} + \xi_{0:k+1} C_0+ \sum_{r=1}^k \xi_{r:k+1}E_r \\
    &= \xi_{0:k+1} C_0 + \sum_{r=1}^{k+1} \xi_{r:k+1}E_r\\
\end{align}
\end{proof}
\noindent\rule{\textwidth}{1pt} 
\begin{lemma} \label{lemma4}
If we further assume that $C_0 = E_0$, then we have for all $k \geq 1$ 
\begin{align}C_k = E_k + \sum_{s=1}^{k} \sum_{q=s}^k \xi_{k-q:k}(s)E_{k-q}\end{align}
\end{lemma}
\begin{proof}
Using the previous lemma, we get 
\begin{align}C_{k} &= E_k+\sum_{s'=1}^k \xi_{0:k}(s')C_0 + \sum_{r=1}^{k-1} \sum_{s=1}^{k-r}\xi_{r:k}(s)E_r \end{align}
Using the assumption $C_0=E_0$, we get 
\begin{align}
    C_{k} &= E_k + \sum_{s'=1}^k \xi_{0:k}(s')E_0 + \sum_{r=1}^{k-1} \sum_{s=1}^{k-r}\xi_{r:k}(s)E_r \\
    &=  E_k+ \sum_{r=0}^{k-1} \sum_{s=1}^{k-r}\xi_{r:k}(s)E_r\\
\end{align}
Now let us put $q = k-r$, we get 
\begin{align}
    C_k &=E_k + \sum_{q=1}^{k} \sum_{s=1}^{q}\xi_{k-q:k}(s)E_{k-q}\\ 
    &= E_k + \sum_{s=1}^{k} \sum_{q=s}^{k}\xi_{k-q:k}(s)E_{k-q}\\
\end{align}
\end{proof}
\noindent\rule{\textwidth}{1pt}
\begin{remark}\label{remark2} First, note that Lemma \ref{lemma4} applies here, since $C_0^{(t)}=E_{0}^{(t)}$, and thus 
\begin{align}C_k^{(t)} = E_k^{(t)} + \sum_{s=1}^{k} \sum_{q=s}^k \xi_{k-q:k}^{(t)}(s)E_{k-q}^{(t)}\end{align}
The idea of Lemma \ref{lemma4} was to regroup all terms with the same number of $F$ factors (where each $F$ contains a Jacobian matrix $J_k$ which contains the connectivity matrix $V$ of the recurrent net). One could roughly perceive the term \begin{align}\sum_{q=s}^k \xi_{k-q:k}^{(t)}(s)E_{k-q}^{(t)}\end{align} as being the term of degree $s$ for $s=1,2,\ldots,k$ and $E_k^{(t)}$ the term of degree $0$. This will allow us to consider the terms $C$ roughly as a polynomial in $V$ and we can look the asymptotic behaviour of each of the coefficients of this polynomial individually. This will then give us a very good understanding on how the distribution of the attention weights are affecting the magnitude of total gradient.\\ 
\end{remark}
\noindent\rule{\textwidth}{1pt}
\begin{proposition} \label{main_prop}
For all $t\geq 1$, and all $k\geq 0$, we have that 
\begin{align}\frac{d s_{t+k}}{d h_t} = \sum_{s=0}^k \bar{\xi}_{o:k}^{(t)}(s)\end{align}
where for all $s\geq 1$, 
\begin{align}\bar{\xi}_{o:k}^{(t)}(s) = \sum_{0\leq i_1<\ldots<i_s<k}F_{k,i_s}^{(t)}\cdot F_{i_s,i_{s-1}}^{(t)}\cdot \ldots \cdot F_{i_2,i_1}^{(t)}\cdot E_{i_1}^{(t)}\end{align}
and where $\bar{\xi}_{o:k}^{(t)}(0) = E_{k}^{(t)}$. With for all $k\geq 0$ we have \begin{align}E_{k}^{(t)} = \frac{\partial s_{t+k}}{\partial h_{t}}\end{align} and for all $k\geq j$ we have \begin{align}F_{k+1,j}^{(t)} = \frac{\partial s_{t+k+1}}{ \partial h_{t+j+1}} \cdot J_{t+j}+ 1_{j=k} \cdot \frac{\partial s_{t+k+1}}{\partial s_{t+k}}\end{align}

\end{proposition}
\begin{proof}
Let $t\geq 1$, and recall that we defined $C_{k}^{(t)} = \frac{d s_{t+k}}{d h_t}$, for all $k\geq 0$.\\ 

As already pointed out, we know that 
$C_0^{(t)} = E_0^{(t)}$ (thus the claim holds for $k=0$).\\

Then by Lemma \ref{lemma4}, we know that for all $k \geq 1$ we have 
\begin{align}
    C_k^{(t)} &= E_k^{(t)} + \sum_{s=1}^{k} \sum_{q=s}^k \xi_{k-q:k}^{(t)}(s)E_{k-q}^{(t)}\\
    &= \bar{\xi}_{o:k}^{(t)}(0) + \sum_{s=1}^{k} \sum_{q=s}^k \sum_{k-q=i_1<\ldots<i_{s+1}=k} F_{k,i_s}^{(t)}\cdot F_{i_s,i_{s-1}}^{(t)} \cdot \ldots \cdot F_{i_2,i_1}^{(t)} \cdot E_{i_1}^{(t)}\\
    &= \bar{\xi}_{o:k}^{(t)}(0) + \sum_{s=1}^{k} \sum_{0\leq i_1<\ldots<i_{s}<k} F_{k,i_s}^{(t)}\cdot F_{i_s,i_{s-1}}^{(t)} \cdot \ldots \cdot F_{i_2,i_1}^{(t)} \cdot E_{i_1}^{(t)}\\
    &= \bar{\xi}_{o:k}^{(t)}(0) + \sum_{s=1}^{k} \bar{\xi}_{o:k}^{(t)}(s)\\
    &= \sum_{s=0}^{k} \bar{\xi}_{o:k}^{(t)}(s)
\end{align}
\end{proof}
\noindent\rule{\textwidth}{1pt}
\begin{remark}\label{remark_assumption}
In what follows the main emphasis will be to calculate the $F_{i,j}^{(t)}$ and $E_{i}^{(t)}$ terms explicitly, since they are the building blocks of the mentioned polynomials in \ref{remark2}.\\

We will assume that 
\begin{align}s_t = f(h_t,c_t)\end{align} with 
\begin{align} c_t = \alpha_{1,t}h_1 + \alpha_{2,t} h_2+ \ldots+ \alpha_{t,t} h_t \end{align}
and 
\begin{align}\alpha_{j,t} = \frac{\exp{(e_{j,t})}}{\sum_{i=1}^t \exp{(e_{i,t})}} \end{align}
where
\begin{align}e_{i,t} = a(s_{t-1},h_i)\end{align}

Let us recall that for all $k\geq 0$ we have \begin{align} E_{k}^{(t)} = \frac{\partial s_{t+k}}{\partial h_{t}} \end{align} and for all $k\geq j$ we have \begin{align}F_{k+1,j}^{(t)} = \frac{\partial s_{t+k+1}}{ \partial h_{t+j+1}} \cdot J_{t+j}+ 1_{j=k} \cdot \frac{\partial s_{t+k+1}}{\partial s_{t+k}} \end{align}

\end{remark}
\noindent\rule{\textwidth}{1pt}
\begin{lemma}\label{expr1} With the assumption of Remark \ref{remark_assumption}, we have that for all $t \geq 2$
\begin{align}\frac{\partial s_{t}}{\partial s_{t-1}} = \partial_2 f(h_t,c_t) \cdot \left(\sum_{i=1}^t \alpha_{i,t} Y_{i,t}\right) \end{align}
where $\partial_2 f$ is the the partial derivative of $f$ with respect to the second variable, and where we define
\begin{align}Y_{i,t} = h_i\cdot  \left(\frac{\partial e_{i,t}}{\partial s_{t-1}}-\sum_{j=1}^t \alpha_{j,t} \cdot \frac{\partial e_{j,t}}{\partial s_{t-1}}\right)\end{align}
\end{lemma}
\begin{proof}
\begin{align}
    \frac{\partial s_{t}}{\partial s_{t-1}} &= \partial_2 f(h_t,c_t)\cdot \frac{\partial c_t}{\partial s_{t-1}}\\
    &= \partial_2 f(h_t,c_t) \cdot \left[\sum_{i=1}^t h_i \cdot \left(\frac{\partial \alpha_{i,t}}{\partial s_{t-1}}\right)\right] \\
    &= \partial_2 f(h_t,c_t) \cdot \left[\sum_{i=1}^t h_i \cdot \left(\sum_{j=1}^t\frac{\partial \alpha_{i,t}}{\partial e_{j,t}} \cdot \frac{\partial e_{j,t}}{\partial s_{t-1}}\right)\right]\\
    &= \partial_2 f(h_t,c_t) \cdot \left[\sum_{i=1}^t h_i \cdot \left( \sum_{j=1}^t \alpha_{i,t} (1_{i=j}-\alpha_{j,t})\cdot \frac{\partial e_{j,t}}{\partial s_{t-1}}\right)\right]\\
    &= \partial_2 f(h_t,c_t) \cdot \left[\sum_{i=1}^t \alpha_{i,t} h_i \left(\frac{\partial e_{i,t}}{\partial s_{t-1}}- \sum_{j=1}^t \alpha_{j,t} \frac{\partial e_{j,t}}{\partial s_{t-1}}\right)\right]\\
    &= \partial_2 f(h_t,c_t) \cdot \left(\sum_{i=1}^t \alpha_{i,t} Y_{i,t}\right)
\end{align}
\end{proof}
\noindent\rule{\textwidth}{1pt}
\begin{lemma} \label{expr2}
With the assumption of Remark \ref{remark_assumption}, we have that for all $k \geq j$:
\begin{align}\frac{\partial s_k}{\partial h_j} = 1_{k=j}\cdot \partial_1 f(h_k,c_k) +\alpha_{j,k} \partial_2 f(h_k,c_k) \cdot \left(I+X_{j,k}\right)\end{align}
where $\partial_1 f$ and $\partial_2 f$ are the partial derivatives of $f$ with respect to the first and second variable, respectively, and where we define 
\begin{align}X_{j,k} = \left(h_j - \sum_{i=1}^k h_i \alpha_{i,k}\right)\cdot \frac{\partial e_{j,k}}{\partial h_j} \end{align}
\end{lemma}
\begin{proof}
\begin{align}
    \frac{\partial s_k}{\partial h_j} &= 1_{k=j}\cdot \partial_1 f(h_k,c_k)\cdot \frac{\partial h_k}{\partial h_k} + \partial_2 f(h_k,c_k) \cdot \frac{\partial c_k}{\partial h_j} \\
    &= 1_{k=j}\cdot \partial_1 f(h_k,c_k) + \partial_2 f(h_k,c_k) \cdot \left[\alpha_{j,k}\cdot I+\sum_{i=1}^k h_i \cdot \frac{\partial \alpha_{i,k}}{\partial h_j}\right]\\
    &= 1_{k=j}\cdot \partial_1 f(h_k,c_k) + \partial_2 f(h_k,c_k) \cdot \left[\alpha_{j,k}\cdot I+\sum_{i=1}^k h_i \cdot \frac{\partial \alpha_{i,k}}{\partial e_{j,k}}\frac{\partial e_{j,k}}{\partial h_j}\right]\\
    &= 1_{k=j}\cdot \partial_1 f(h_k,c_k) + \partial_2 f(h_k,c_k) \cdot \left[\alpha_{j,k}\cdot I+\sum_{i=1}^k h_i \cdot \alpha_{i,k}(1_{i=j}-\alpha_{j,k})\frac{\partial e_{j,k}}{\partial h_j}\right]\\
    &= 1_{k=j}\cdot \partial_1 f(h_k,c_k) + \partial_2 f(h_k,c_k) \cdot \left[\alpha_{j,k}\cdot I+\left(h_j \alpha_{j,k}- \alpha_{j,k}\sum_{i=1}^k h_i \cdot \alpha_{i,k}\right)\frac{\partial e_{j,k}}{\partial h_j}\right]\\
    &= 1_{k=j}\cdot \partial_1 f(h_k,c_k) +\alpha_{j,k} \partial_2 f(h_k,c_k) \cdot \left[I+\left(h_j-\sum_{i=1}^k h_i \cdot \alpha_{i,k}\right)\frac{\partial e_{j,k}}{\partial h_j}\right]\\
    &= 1_{k=j}\cdot \partial_1 f(h_k,c_k) +\alpha_{j,k} \partial_2 f(h_k,c_k) \cdot \left(I+X_{j,k}\right)
\end{align}
\end{proof}
\noindent\rule{\textwidth}{1pt} 
\begin{corollary}\label{main_corollary}
With the assumption of Remark \ref{remark_assumption}, and the notations of lemma \ref{expr1} and \ref{expr2}, we have for all $k'\geq 0$,
\begin{align} E_{k'}^{(t)} = 1_{k'=0} \partial_1 f(h_t,c_t) + \alpha_{t,t+k'} \partial_2 f(h_{t+k'},c_{t+k'}) \cdot \left[I+X_{t,t+k'}\right] \end{align}
and for all $k\geq j$,
\begin{align}
    F_{k+1,j}^{(t)} &= \alpha_{t+j+1,t+k+1} \cdot \partial_2 f(h_{t+k+1},c_{t+k+1})\cdot \left[I+ X_{t+j+1,t+k+1}\right]\cdot J_{t+j}\\
    &+ 1_{k=j}\cdot \left(\partial_1 f(h_{t+k+1},c_{t+k+1})J_{t+j}+ \partial_2 f(h_{t+k+1},c_{t+k+1})\cdot \left[\sum_{i=1}^{t+k+1}\alpha_{i,t+k+1}Y_{i,t+k+1}\right]\right)
\end{align}
\end{corollary}
\begin{proof}
Applying lemma \ref{expr2}, we get that for all $k\geq 0$,
\begin{align}
    E_{k'}^{(t)} &= \frac{\partial s_{t+k'}}{\partial h_t}\\
    &= 1_{k'=0}\cdot \partial_1 f(h_t,c_t) + \alpha_{t,t+k'} \cdot \partial_2 f(h_{t+k'},c_{t+k'})\cdot \left[I+X_{t,t+k'}\right]\\
\end{align}
and then by applying lemma \ref{expr1} and \ref{expr2}, we get that for all $k \geq j$,

\begin{align}
   F_{k+1,j}^{(t)} &= \frac{\partial s_{t+k+1}}{ \partial h_{t+j+1}} \cdot J_{t+j}+ 1_{j=k} \cdot \frac{\partial s_{t+k+1}}{\partial s_{t+k}}\\
   &= [1_{k=j} \partial_1 f(h_{t+k+1},c_{t+k+1})\\
   &+ \alpha_{t+j+1,t+k+1}\partial_2 f(h_{t+k+1},c_{t+k+1})\cdot(I+X_{t+j+1,t+k+1})]\cdot J_{t+j}\\ 
   &+ 1_{k=j} \cdot \partial_2 f(h_{t+k+1},c_{t+k+1})\cdot \left(\sum_{i=1}^{t+k+1} \alpha_{i,t+k+1} Y_{i,t+k+1}\right)\\
   &= \alpha_{t+j+1,t+k+1} \cdot \partial_2 f(h_{t+k+1},c_{t+k+1})\cdot \left[I+ X_{t+j+1,t+k+1}\right]\cdot J_{t+j}\\
    &+ 1_{k=j}\cdot \left(\partial_1 f(h_{t+j+1},c_{t+k+1})J_{t+j}+ \partial_2 f(h_{t+k+1},c_{t+k+1})\cdot \left[\sum_{i=1}^{t+k+1}\alpha_{i,t+k+1}Y_{i,t+k+1}\right]\right)
\end{align}
\end{proof}
\noindent\rule{\textwidth}{1pt}
\begin{proposition}\label{prop_D}
We can rewrite for all $k' \geq 0$ and all $k \geq j \geq 0$
\begin{align}
    E_{k'}^{(t)} &= \alpha_{t,t+k'}\cdot \tilde{D}_{k',0}^{(t)}+ 1_{k'=0} \tilde{R}_0^{(t)}\\
    F_{k+1,j}^{(t)} &= \alpha_{t+j+1,t+k+1} \cdot D_{k+1,j}^{(t)} + 1_{k=j} \cdot R_{k+1}^{(t)}
\end{align}
where 
\begin{align}
    D_{k+1,j+1}^{(t)} &= \partial_2 f(h_{t+k+1},c_{t+k+1}) \cdot [I+ X_{t+j+1,t+k+1}]\cdot J_{t+j}\\
    R_{k+1}^{(t)} &= \partial_1 f(h_{t+k+1}, c_{t+k+1})\cdot J_{t+k} + \partial_2 f(h_{t+k+1},c_{t+k+1})\cdot [\sum_{i=1}^{t+k+1} \alpha_{i,t+k+1} Y_{i,t+k+1}]\\
    \tilde{D}_{k'}^{(t)} &= \partial_2 f(h_{t+k'},c_{t+k'}) \cdot [I+ X_{t,t+k'}]\\
    \tilde{R}_0^{(t)} &= \partial_1 f(h_t,c_t)
\end{align}
while $X_{i,i'}$ and $Y_{i,i'}$ are defined as in lemma \ref{expr1} and \ref{expr2}.
\end{proposition}
\begin{proof}
Follows straight from Corollary \ref{main_corollary}.
\end{proof}
\noindent\rule{\textwidth}{1pt}
\begin{remark}
If we are further assuming that \begin{align}s_t = f(h_t,c_t) = h_t+c_t \end{align} then for all $k\geq 0$, we have 
\begin{align}E_{k}^{(t)} = 1_{k=0}\cdot I + \alpha_{t,t+k} \cdot \left[I+X_{t,t+k}\right]\end{align}
and for all $k\geq j$, we have
\begin{align}
    F_{k+1,j}^{(t)} &= \alpha_{t+j+1,t+k+1} \cdot \left[I+ X_{t+j+1,t+k+1}\right]\cdot J_{t+j} + 1_{k=j}\cdot \left(J_{t+j}+ \left[\sum_{i=1}^{t+k+1}\alpha_{i,t+k+1}Y_{i,t+k+1}\right]\right)
\end{align}

\end{remark}
\begin{proof} This follows directly form corollary \ref{main_corollary} and the observation that \begin{align}\partial_1 f(h_t,c_t) = \partial_2 f(h_t,c_t) = I\end{align}
\end{proof}
\noindent\rule{\textwidth}{1pt}
\begin{remark} \label{dima_attention}
If we are further assuming that \begin{align}e_{j,t}=a(s_{t-1},h_j)= v_a^{T}\cdot\tanh{(W_a s_{t-1}+ U_a h_j)} \end{align}
as done in \cite{attention}, we get that 
\begin{align}
    \frac{\partial e_{j,t}}{\partial h_j} &= v_a^T\cdot \textrm{diag}[1-\tanh^2{(W_a s_{t-1}+U_a h_j)}] \cdot U_a
\end{align}
and
\begin{align}
    \frac{\partial e_{j,t}}{\partial s_{t-1}}&= v_a^T\cdot \textrm{diag}[1-\tanh^2{(W_a s_{t-1}+U_a h_j)}] \cdot W_a
\end{align}
which we can plug into the definitions of $X_{j,k}$ and $Y_{j,k}$ to get explicit expressions for matrices $E_{k'}^{(t)}$ and $F_{k+1,j}^{(t)}$.
\end{remark}
\noindent\rule{\textwidth}{1pt}
\begin{lemma} \label{gradient_alignment_weights}
If, with the assumptions Remark \ref{remark_assumption}, we assume that for all $i,t\geq 1$, we have $e_{i,t} = a(s_{t-1},h_i, \theta)$ depending on some parameter $\theta \in \mathbb{R}^{N\times M}$, then we have
\begin{align}\frac{dL}{d\theta} = \sum_{j,t} \alpha_{j,t} \cdot \frac{dL}{d s_t} \cdot \partial_2 f(h_t,c_t) \cdot h_j \cdot \left[\sum_{i}(1_{i=j}-\alpha_{i,t})\cdot \frac{\partial e_{i,t}}{\partial \theta}\right]\end{align}
\end{lemma}
\begin{proof} If we denote $\theta^{(i,t)}$ to be the parameter for $e_{i,t}$, then we have 
\begin{align}
    \frac{dL}{d\theta} &= \sum_{i,t}\frac{dL}{d\theta^{(i,t)}}\\
    &= \sum_{i,j,t} \frac{dL}{d \alpha_{j,t}}\cdot \frac{\partial \alpha_{j,t}}{\partial e_{i,t}}\cdot \frac{\partial e_{i,t}}{\partial \theta^{(i,t)}}\\
    &= \sum_{i,j,t} \alpha_{j,t}(1_{i=j}-\alpha_{i,t})\cdot \frac{dL}{d \alpha_{j,t}}\cdot \frac{\partial e_{i,t}}{\partial \theta}
\end{align}
where \begin{align}\frac{dL}{d\alpha_{j,t}}= \frac{dL}{ d s_t}\cdot \frac{\partial s_t}{\partial c_t} \cdot \frac{\partial c_t}{\partial \alpha_{j,t}} = \frac{dL}{ d s_t}\cdot \partial_2 f(h_t,c_t) \cdot h_j\end{align}
Hence 
\begin{align}
    \frac{dL}{d\theta} &= \sum_{i,j,t} \alpha_{j,t}(1_{i=j}-\alpha_{i,t})\cdot \frac{dL}{ d s_t}\cdot \partial_2 f(h_t,c_t) \cdot h_j\cdot \frac{\partial e_{i,t}}{\partial \theta}\\
    &= \sum_{j,t} \alpha_{j,t} \cdot \frac{dL}{d s_t} \cdot \partial_2 f(h_t,c_t) \cdot h_j \cdot \left[\sum_{i}(1_{i=j}-\alpha_{i,t})\cdot \frac{\partial e_{i,t}}{\partial \theta}\right]
\end{align}
\end{proof}
\noindent\rule{\textwidth}{1pt} 

\begin{lemma} \label{gradient_connectivity}
Let us recall that for all $t\geq 0$, we have  \begin{align} h_{t+1}=\phi(\underbrace{V s_{t}+ U x_{t+1}+b}_{=a_t}) \end{align}
where $\phi$ is a non-linear activation function, $V\in \mathbb{R}^{n\times n}$, $U \in \mathbb{R}^{n\times m}$ and $b\in \mathbb{R}^n$. Then we have that
\begin{align}\left[\frac{d L}{ d V}, \frac{d L}{d U}, \frac{d L}{d b}\right] = \sum_{t=1}^T [s_{t-1},x_t,1]\cdot \frac{d L}{d h_{t}}\cdot \textrm{diag}(\phi'(a_t))\end{align}
\end{lemma}
\begin{proof} Let us denote $V^{(t)}, U^{(t)}, b^{(t)}$ the matrices $V,U,b$ of $a_{t-1}$ respectively, then 
\begin{align}
    \left[\frac{d L}{d V},\frac{d L}{d U}, \frac{d L}{d b} \right] &= \sum_{t} \left[\frac{d L}{d V^{(t)}},\frac{d L}{d U^{(t)}}, \frac{d L}{d b^{(t)}}\right]\\
    &= \sum_t [s_{t-1},x_t,1]\cdot \frac{d L}{d a_{t-1}}\\
    &= \sum_t [s_{t-1},x_t,1]\cdot \frac{d L}{d h_t}\cdot \frac{d h_t}{d a_{t-1}}\\
    &= \sum_t [s_{t-1},x_t,1]\cdot \frac{d L}{d h_t}\cdot \textrm{diag}(\phi'(a_{t-1}))
\end{align}    
\end{proof}
\noindent\rule{\textwidth}{1pt} 
\begin{remark} \label{relevant_states_remark}
Combining the fact that $\frac{dL}{dh_t} = \frac{dL}{ds_T}\frac{d s_T}{d h_t}$, the results from propositions \ref{main_prop} and \ref{prop_D}, with lemma \ref{gradient_connectivity}, we see that attention weights $\alpha_{i,t}$ which are very close to $0$, do not contribute to the gradient and the learning of $V,U$ and $b$. \\

Similarly, it follows directly from lemma \ref{gradient_alignment_weights}, that attention weights $\alpha_{i,t}$ which are very close to $0$, do not contribute to the gradient and the learning of any parameters $\theta$ of the alignment function $e_{i,t} = a(s_{t-1},h_i,\theta)$. In case we have an alignment function as in remark \ref{dima_attention}, these parameters are $W_a, U_a$ and $v_a$.\\

If we have the case where one state $h_i$ is such that all attention weights $\alpha_{i,t} \approx 0$ for all $t\geq i$, then we can see that $h_i$ does not contribute to the gradient and learning to any parameters be it parameters from the recurrence or the alignment function.\\

In practice we have observed that in the majority of tasks, most states $h_i$ fall in either of two categories:
\begin{itemize}
    \item $\alpha_{i,t}$ is sufficiently bounded away from $0$ for most $t\geq i$, and thus contributes to learning. This is what we call a "relevant state". 
    \item $\alpha_{i,t} \approx 0$ for almost all $t\geq i$, and thus doesn't contribute much to learning, and the gradient can be approximated by assuming $\alpha_{i,t} = 0$ for all $t\geq i$. This is what he call a "non-relevant state".
\end{itemize}
This observation is what lead us to the intuition that we can approximate the gradient, by decomposing it via proposition \ref{main_prop}, into gradient paths involving only skip connections between "relevant states". 
\end{remark}
\noindent\rule{\textwidth}{1pt}
\subsection{Uniform attention case}
\label{uniform_attention_case}
\begin{remark} \label{gen_assumptions}
In this subsection, we are going to assume:
\begin{itemize}
    \item no non-linearity in the hidden-to-hidden connection: $J_t = V$ for all $t$.
    \item all assumptions from Remark \ref{remark_assumption}.
    \item uniform attention: $\alpha_{i,t}=1/t$ for all $t\geq 1$.
\end{itemize}

\end{remark}
\noindent\rule{\textwidth}{1pt} 
\subsubsection{Overview}
\begin{remark}\label{unif_assumptions} Recalling corollary \ref{main_corollary}, together the main proposition \ref{main_prop} form last section, we can hope to simplify these expressions using the new assumptions from the previous remark \ref{gen_assumptions}. Recalling expression from lemma \ref{expr1} and \ref{expr2}:
\begin{align}
    X_{j,t} &= \left(h_j - \sum_{i=1}^t h_i \alpha_{i,t}\right)\cdot \frac{\partial e_{j,t}}{\partial h_j}\\
    &= \left(h_j - \frac{1}{t}\sum_{i=1}^t h_i \right)\cdot \frac{\partial e_{j,t}}{\partial h_j}
\end{align}
Hence, for our calculations we are going to assume that $\left(h_j - \frac{1}{t}\sum_{i=1}^t h_i \right) \approx 0$, and thus $X_{j,t} \approx 0$ for all $1\leq j \leq t$. Similarly,
\begin{align}
    \sum_{i=1}^t \alpha_{i,t}Y_{i,t} &=
    \sum_{i=1}^t \alpha_{i,t} h_i\cdot \left(\frac{\partial e_{i,t}}{\partial s_{t-1}}-\sum_{j=1}^t \alpha_{j,t} \cdot \frac{\partial e_{j,t}}{\partial s_{t-1}}\right)\\
    &= \frac{1}{t}\sum_{i=1}^t  h_i\cdot \left(\frac{\partial e_{i,t}}{\partial s_{t-1}}-\sum_{j=1}^t \frac{1}{t} \cdot \frac{\partial e_{j,t}}{\partial s_{t-1}}\right)\\
    &= \frac{1}{t}\sum_{i=1}^t  h_i\cdot \frac{\partial e_{i,t}}{\partial s_{t-1}}
    - \frac{1}{t}\sum_{j=1}^t \left(\frac{1}{t}\sum_{i=1}^t h_i\right)\cdot \frac{\partial e_{j,t}}{\partial s_{t-1}}\\
    &= \frac{1}{t}\sum_{i=1}^t  h_i\cdot \frac{\partial e_{i,t}}{\partial s_{t-1}}
    - \frac{1}{t}\sum_{i=1}^t \left(\frac{1}{t}\sum_{j=1}^t h_j\right)\cdot \frac{\partial e_{i,t}}{\partial s_{t-1}}\\
    &= \frac{1}{t}\sum_{i=1}^t \left(h_i - \frac{1}{t}\sum_{j=1}^t h_j\right) \cdot \frac{\partial e_{i,t}}{\partial s_{t-1}}\\
    &\approx 0 
\end{align}
Recalling the expression from corollary \ref{main_corollary} and that $f(h_t,c_t)= h_t+c_t$ by remark \ref{remark_assumption}, and that $J_t = V$ for all $t$, this will give for all $k'\geq 0$
\begin{align}
    E_{k'}^{(t)} &= \left(\frac{1}{t+k'}+1_{k'=0}\right)\cdot \textrm{I}
\end{align}
and for all $k\geq j$, we get
\begin{align}
 F_{k+1,j}^{(t)} &= \left(\frac{1}{t+k+1} + 1_{k=j}\right)\cdot V   
\end{align}
Hence by recalling proposition \ref{main_prop}, the main expression of interest becomes 
\begin{align}\frac{d s_{t+k}}{d h_t} = \sum_{s=0}^k \bar{\xi}_{0:k}^{(t)}(s) = \sum_{s=0}^k V^s \cdot \chi_{0:k}^{(t)}(s)\end{align}

where 
\begin{align}\chi_{0:k}^{(t)}(s)&\myeq \sum_{0\leq i_1<\ldots <i_s<k}\left(\frac{1}{t+k}+ 1_{k-i_{s}=1}\right)\cdot\left(\frac{1}{t+i_s}+ 1_{i_{s}-i_{s-1}=1}\right)\cdot \ldots\\ 
&\ldots\cdot \left(\frac{1}{t+i_2}+1_{i_2-i_1=1}\right)\cdot \left(\frac{1}{t+i_1}+1_{i_1=0}\right)\end{align}
\end{remark}
%here
\noindent\rule{\textwidth}{1pt}
\begin{remark}
The goal is thus to have a good estimation of the terms \begin{align}\chi_{0:k}^{(t)}(s)\end{align} in order to then find an asymptotic estimation for \begin{align}\frac{d s_{t+k}}{d h_t} = \sum_{s=0}^k V^s \cdot \chi_{0:k}^{(t)}(s) \end{align} as $k\rightarrow \infty$. In order to do so, we will adopt the following strategy:\\

\paragraph{Step 1.} Estimate the expression 
\begin{align}\omega_{l:k}^{(t)}(s) \myeq  \sum_{l\leq i_1<\ldots<i_s<k} \frac{1}{t+i_s}\cdot \frac{1}{t+i_{s-1}}\cdot \ldots \cdot \frac{1}{t+i_2}\cdot \frac{1}{t+i_1}\end{align}
for all $s\geq 1$. This will be done in sub-subsection \ref{subsub_omega}.\\

\paragraph{Step2.} Estimate the expression

\begin{align}\theta_{l:k}^{(t)}(s) &\myeq \sum_{l\leq i_1<\ldots<i_s<k} \left(\frac{1}{t+i_s}+1_{i_s-i_{s-1}=1}\right)\cdot \left(\frac{1}{t+i_{s-1}}+1_{i_{s-1}-i_{s-2}=1}\right)\cdot \ldots\\ 
&\ldots \cdot \left(\frac{1}{t+i_2}+1_{i_2-i_1=1}\right)\cdot \left(\frac{1}{t+i_1}+1_{i_1=0}\right)\end{align}
for all $s\geq 1$, because as we will see the expression  $\theta_{l:k}^{(t)}(s)$ can be decomposed into $\omega_{l':k'}^{(t)}(s')$ expressions for $s\geq s'\geq 1$. This will be done in sub-subsection \ref{subsub_theta}.\\

\paragraph{Step 3.} The final step will consist in putting the results from the two previous sub-subsections together, and getting a final asymptotic estimate for $\frac{d s_{t+k}}{d h_t}$ as $k \rightarrow \infty$, by noting that 

\begin{align}\chi_{0:k}^{(t)}(s) &= \frac{1}{t+k}\cdot \theta_{0:k}^{(t)}(s)+\frac{1}{t+k-1}\cdot\theta_{0:k-1}^{(t)}(s-1)+\ldots \\
&\ldots+\frac{1}{t+k-s+1}\cdot \theta_{0:k-s+1}^{(t)}(1) + \frac{1}{t+k-s}+1_{k=s} \end{align}
This will be treated in sub-subsection \ref{subsub_all}.
\end{remark}
\noindent\rule{\textwidth}{1pt} 
\subsubsection{Estimating $\omega$} \label{subsub_omega}
\begin{remark}
In this sub-subsection we are going to estimate $\omega_{0:k}^{(t)}(s)$, which is a sum of products of $s$ distinct factors. The idea will be to start from the expression 
\begin{align}\left(\frac{1}{t}+\frac{1}{t+1}+\ldots + \frac{1}{t+k-1}\right)^s \end{align} and substract all products containing at least two identical factors, followed by a division by $s!$.\\

This approach will be similar in spirit to the inclusion-exclusion principle, with the only difference that the desired term will not computed directly, but instead one first establishes a recursive formula using $\omega_{0:k}^{(t)}(s')$ with $s'\leq s$.\\ 

Solving this recursive formula will enable us to express $\omega_{0:k}^{(t)}(s)$ only in terms of $(\frac{1}{t}+\frac{1}{t+1}+\ldots + \frac{1}{t+k-1})$. In fact, $\omega_{0:k}^{(t)}(s)$ will be a polynomial of degree $s$ in $(\frac{1}{t}+\frac{1}{t+1}+\ldots + \frac{1}{t+k-1})$.\\

We adopt this approach, because we have a very good estimate for \begin{align}\frac{1}{t}+\frac{1}{t+1}+\ldots + \frac{1}{t+k-1}\end{align} Namely, we know that for all $n$, we have 

\begin{align} 1+\frac{1}{2}+\ldots+\frac{1}{n-1}+\frac{1}{n} = \ln{n} + \gamma + \varepsilon_n \leq \ln{n} + 1\end{align}
where $\gamma >\frac{1}{2}$ is the Euler-Mascheroni constant and $\varepsilon_n$ behaves asymptotically as $\frac{1}{2n}$. In other words, 
\begin{align}
    \frac{1}{t}+\frac{1}{t+1}+\ldots + \frac{1}{t+k-1} &= \ln{\left(\frac{t+k-1}{t-1}\right)}+\varepsilon_{t+k-1}-\varepsilon_{t-1}\\ 
    &= \ln{\left[\frac{t+k-1}{t-1}\cdot \exp{(\varepsilon_{t+k-1}-\varepsilon_{t-1})}\right]}\\ 
    &=  \ln{\beta_{t-1,t+k-1}}
\end{align}
where $\beta_{l,l'} \myeq \frac{l'}{l}\cdot \exp{(\varepsilon_{l'}-\varepsilon_{l})}$. In order to reinforce the intuition here, let us imagine that $T=t+k$, then \begin{align}\ln{\beta_{t-1,t+k-1}} \sim \ln{T}\end{align} as $T\rightarrow \infty$. Hence we should expect $\omega_{0:k}^{(t)}(s)$ to behave asymptotically as a polynomial of degree $s$ in $\ln{T}$.

Let us emphasize that we would like to express $\omega_{0:k}^{(t)}(s)$ with as much precision as possible (i.e. not omitting the monomials in $\ln{T}$ of degree less than $s$), since we would like to later on use this estimate in subsequent steps when summing multiple $\omega_{0:k}^{(t)}(s)$ terms over $s$.\\ 

In order to further ease notation, we will simply write $\omega(s)$ for $\omega_{0:k}^{(t)}(s)$, whenever there is no ambiguity.\\

Finally, for this sub-subsection only we will use the following notation

\begin{align}S_l \myeq \frac{1}{t^l}+\frac{1}{(t+1)^l}+ \ldots + \frac{1}{(t+k-1)^l}\end{align}
for all $l\geq 1$, and keeping in mind that $S_l$ converges as $k\rightarrow \infty$, for all $l\geq 2$.\\ 

\end{remark}
\noindent\rule{\textwidth}{1pt} 

\begin{remark}
Let us now build a first intuition on how to apply an inclusion-exclusion-like principle in order to calculate $\omega(s)$ for small $s$.\\

For $\underline{s=1}$: \begin{align}\omega(1) = S_1\end{align}
For $\underline{s=2}$: \begin{align}\omega(2)=\frac{1}{2!}\left(S_1^2-S_2\right)\end{align}
Here we expand $S_1^2$, then substract the sum of products of doubles $S_2$, followed by a division of $2!=2$ to divide out the number of permutations. \\

For $\underline{s=3}$: first we need to substract the sum of products of triples $S_3$, and then the sum products where exactly two factors are identical $S_2\cdot \omega(1)-S_3$. The latter appears ${3\choose 2,1} = \frac{3!}{2!1!}=3$ times in the expansion of $S_1^3$. Similarly, we need to divide out the number of permutations $3!$. Hence 
\begin{align}\omega(3) = \frac{1}{3!} \left[S_1^3-S_3-3\cdot (S_2\cdot \omega(1)-S_3)\right] = \frac{S_1^3}{3!}-\frac{1}{2}S_2\cdot \omega(1)+\frac{1}{3}S_3\end{align}

Let us form now on denote $(3)$ for the sum of products of triples, and $(2,1)$ the sum of products where exactly two factors are the same.\\ 

More generally we would denote \begin{align}(j_1,j_2,\ldots,j_k)\end{align} with $j_1\geq j_2 \geq \ldots \geq j_k\geq 1$, to denote the sum of products where one factor appears exactly $j_1$ times, another factor (distinct from the previous one!) appears exactly $j_2$ times, and another factor (distinct from the previous two!) appears exactly $j_3$ times, etc. This leaves us with exactly $k$ distinct factors each having multiplicity $j_1,j_2,\ldots,j_k$ respectively. This sum appears with 

\begin{align}{s\choose j_1,j_2,\ldots,j_k}= \frac{s!}{j_1!\cdot j_2!\cdot \ldots\cdot j_k!}\end{align}
repetitions in the expansion of $S_1^s$, where $s=j_1+j_2+\ldots+j_k$.\\

For $\underline{s=4}$: when expanding $S_1^4$, we need to take into account 
\begin{itemize}
    \item $(4)=S_4$ with ${4\choose 4} = \frac{4!}{4!} = 1$ repetition. 
    \item $(3,1) = S_3\cdot \omega(1)-S_4$ with ${4\choose 3,1} = \frac {4!}{3!\cdot 1!} = 4$ repetitions.
    \item $(2,2) = S_2^2-S_4$ with ${4 \choose 2,2} = \frac{4!}{2!\cdot 2!}= 6$ repetitions.
    \item $(2,1,1) = S_2\cdot \omega(2)- (3,1) = S_2\cdot \omega(2)- S_3\cdot \omega(1)+S_4$ with ${4 \choose 2,1,1} = \frac{4!}{2!\cdot 1!\cdot 1!}= 12$ repetitions.
\end{itemize}

Hence we get 
\begin{align}
    \omega(4) &= \frac{1}{4!}[S_1^4-S_4-4\cdot (S_3\cdot \omega(1)-S_4)-6\cdot(S_2^2-S_4) \\
    &-12\cdot(S_2\cdot \omega(2)- S_3\cdot \omega(1)+S_4)]\\
    &= \frac{1}{4!}[S_1^4-4\cdot S_3\cdot \omega(1)+4\cdot S_4-6\cdot S_2^2+6\cdot S_4-12 \cdot S_2 \omega(2) \\
    &+12\cdot S_3\cdot \omega(1)-12\cdot S_4 -S_4]\\
    &= \frac{1}{4!}\left[S_1^4-12\cdot S_2\cdot \omega(2)+ 8\cdot S_3\cdot \omega(1)-3\cdot (S_4+S_2^2)\right]\\
    &= \frac{S_1^4}{4!} -\frac{S_2}{2}\omega(2)+ \frac{S_3}{3}\omega(1)-\frac{(S_4+2\cdot S_2^2)}{8}
\end{align}
Notice how, as we progress with higher values of $s$, we build a recursive formula in $\omega(s')$ with $s'\leq s$.

\paragraph{Intuition.} Note that the coefficient of $\omega(2)$ for $s=4$, is the same as the coefficient for $\omega(1)$ for $s=3$, and is the same as the 'constant term' for $s=2$. Similarly, the coefficient of $\omega(1)$ for $s=4$ is the same as the 'constant term' for $s=3$. (By convention here, we don't consider the terms $\frac{S_1^s}{s!}$ to not be part of the 'constant term'.)

Hence, in the recursive formula for $\omega(s)$, we would expect the coefficient of $\omega(s')$  with $s'<s$ to be equal to the 'constant term' in the formula for $\omega(s-s')$. 

\paragraph{Notation.} For all $s> l\geq 0$, let us denote $\delta_{s,l}$ to be the coefficient of the term $\omega(l)$ in the recursive formula for $\omega(s)$. By convention, we denote $\delta_{s,0}$ for the 'constant term' in the recursive formula for $\omega(s)$. Hence for all $s\geq 1$, we have

\begin{align}\omega(s) = \frac{S_1^s}{s!}+ \delta_{s,s-1}\cdot \omega(s-1)+\delta_{s,s-2}\cdot \omega(s-2)+\ldots+ \delta_{s,1}\cdot \omega(1)+ \delta_{s,0}\end{align}

\paragraph{Hypothesis.} The hypothesis will thus rewrite as 

\begin{align}\delta_{s,l} = \delta_{s-l,0}\end{align}
for all $s>l\geq 0$, which will prove by induction on $s$ in the next lemma.\\

\end{remark}
%here
\noindent\rule{\textwidth}{1pt}
\begin{lemma}
Let $s\geq 1$. Then 
\begin{align}\omega(s) = \frac{S_1^s}{s!}+ \delta_{1,0}\cdot \omega(s-1)+\delta_{2,0}\cdot \omega(s-2)+\ldots+ \delta_{s-1,0}\cdot \omega(1)+ \delta_{s,0} \end{align}
\end{lemma}
\begin{proof}
Let us prove by induction on $s$ that for all $s> l\geq 0$, we have \begin{align}\delta_{s,l}= \delta_{s-l,0}\end{align}.
We already verified the cases $s=1,2,3,4$ in the previous remark. Thus let us suppose the induction hypothesis is true for $s$, and consider the mapping 
\begin{align}\Upsilon: (j_1,j_2,\ldots,j_k) \mapsto (j_1,j_2,\ldots,j_k,1)\end{align}
where $j_1\geq j_2\geq\ldots \geq j_k\geq 1$ and $s= j_1+j_2+\ldots+j_k$, mapping a partition of $s$ onto a partition of $s+1$.

If we suppose that $(j_1,j_2,\ldots,j_k)$ consists of exactly $r$ $1$'s, then we can write 
\begin{align}(j_1,j_2,\ldots,j_k) = c_r \cdot \omega(r)+ c_{r-1}\cdot \omega(r-1)+\ldots +c_1\cdot  \omega(1)+ c_0\end{align} for some coefficients $c_r,c_{r-1},\ldots,c_1,c_0$, and with \begin{align}{s \choose j_1,j_2,\ldots,j_k}= \frac{s!}{j_1!\cdot j_2\cdot \ldots \cdot j_k!}\end{align} repetitions in the expansion of $S_1^s$.

The contribution of $(j_1,j_2,\ldots,j_k)$ to the coefficient $\delta_{s,r'}$ of $\omega(r')$ with $r'\leq r<s$, in the final recursive formula of $\omega(s)$ will be \begin{align}\frac{c_{r'}}{j_1!\cdot j_2!\cdot \ldots \cdot j_k!}\end{align} (keeping in mind that we are dividing by $s!$ after having done all the substractions from $S_1^s$). 

Meanwhile, 
\begin{align}(j_1,j_2,\ldots,j_k,1) = c_r\cdot  \omega(r+1) + c_{r-1}\cdot \omega(r)+ \ldots + c_1 \cdot \omega(2)+ c_0\cdot \omega(1)+ \tilde{c_0} \end{align}
for some coefficient $\tilde{c_0}$, with \begin{align}{s+1 \choose j_1,j_2,\ldots,j_k,1}= \frac{(s+1)!}{j_1!\cdot j_2\cdot \ldots \cdot j_k!}\end{align} repetitions in the expansion of $S_1^{s+1}$.

The contribution of $(j_1,j_2,\ldots,j_k,1)$ to the coefficient $\delta_{s+1,r'+1}$ of $\omega(r'+1)$ with $r'\leq r<s$, in the final recursive formula of $\omega(s+1)$ will be \begin{align}\frac{c_{r'}}{j_1!\cdot j_2!\cdot \ldots \cdot j_k!}\end{align} (keeping in mind that we are dividing by $(s+1)!$ after having done all the substractions from $S_1^{s+1}$).

Conversely, the coefficient $\delta_{s+1,r'+1}$ only receives contributions from partitions of $(s+1)$ having at least $(r'+1)$ $1$'s, which correspond exactly to the contributions from the partitions of $s$ having at least $r'$ $1$'s. Hence 
\begin{align}\delta_{s+1,r'+1} = \delta_{s,r'}\end{align} Then by the induction hypothesis, we have $\delta_{s,r'} = \delta_{s-r',0}$. In other words
\begin{align}\delta_{s+1,r'+1} = \delta_{s-r',0}\end{align}
which completes the proof by induction. \\
\end{proof}
\noindent\rule{\textwidth}{1pt}
\begin{remark}
Note that all the coefficients $\delta_{s,l}$ consist of linear combination of products with factors equal to $S_j$ with $j\geq 2$, which are known to converge as $T\rightarrow \infty$. Thus those can be considered constants when doing an asymptotic analysis in the subsequent sub-subsections. Also note that $\delta_{s,s-1}= \delta_{1,0}=0$.\\ 
\end{remark}
\noindent\rule{\textwidth}{1pt} 
\begin{proposition} \label{omega_lemma}
For all $s\geq 1$, we have 
\begin{align}\omega(s)= \sum_{r=0}^s \psi_{s-r}\frac{S_1^r}{r!} \end{align}
where for $l \geq 2$
\begin{align}\psi_l \myeq \sum_{k=1}^{l-1} \sum_{(j_1,j_2,\ldots,j_k)\in \Psi_{l,k}} \delta_{j_1,0}\cdot \ldots \cdot \delta_{j_k,0}\end{align}
with \begin{align}\Psi_{l,k} \myeq \{(j_1,j_2,\ldots,j_k) \textrm{ with } j_1\geq \ldots \geq j_k >1 \textrm{ and } j_1+\ldots+j_k =l\}\end{align}
and where we define $\psi_0=1$ and $\psi_1=0$. 
\end{proposition}
\begin{proof}
For $l\geq 2$, we have 
\begin{align}
    \psi_l &= \sum_{k=1}^{l-1} \sum_{(j_1,j_2,\ldots,j_k)\in \Psi_{l,k}} \delta_{j_1,0}\cdot \ldots \cdot \delta_{j_k,0}\\
    &= \delta_{l,0} + \sum_{k=1}^{l-1} \left(\sum_{j=2}^{l-2} \sum_{(j_2,\ldots,j_k)\in \Psi_{l-j,k-1}} \delta_{j,0}\cdot \delta_{j_2,0}\cdot \ldots \cdot \delta_{j_k,0} \right)\\
    &= \delta_{l,0} + \sum_{j=2}^{l-2}  \left(\sum_{k=1}^{l-j} \sum_{(j_2,\ldots,j_k)\in \Psi_{l-j,k-1}} \delta_{j,0}\cdot \delta_{j_2,0}\cdot \ldots \cdot \delta_{j_k,0} \right)\\
    &= \delta_{l,0} + \sum_{j=2}^{l-2} \delta_{j,0}\cdot \left(\sum_{k=1}^{l-j} \sum_{(j_2,\ldots,j_k)\in \Psi_{l-j,k-1}}  \delta_{j_2,0}\cdot \ldots \cdot \delta_{j_k,0} \right)\\
    &= \delta_{l,0} + \sum_{j=2}^{l-2} \delta_{j,0}\cdot \psi_{l-j}\\
    &= \sum_{j=1}^{l} \delta_{j,0}\cdot \psi_{l-j}\\
\end{align}
In other words, we have shown that for all $l\geq 2$,
\begin{align}\psi_l = \sum_{j=0}^{l-1} \delta_{l-j} \psi_j\end{align}
Let us now prove the proposition by induction on $s$.

The case $\underline{s=1}$ is trivial by the definition of $\psi_0$ and $\psi_1$.

Let us now assume the formula is true for $s$, and let us prove it for $s+1$. By the previous lemma \ref{omega_lemma}, we know that 
\begin{align}
    \omega(s+1) &= \frac{S_1^{s+1}}{(s+1)!}+ \sum_{l=1}^s \delta_{s+1-l,0} \cdot \omega(l) + \delta_{s+1,0}\\
    &= \frac{S_1^{s+1}}{(s+1)!}+ \sum_{l=1}^s \delta_{s+1-l,0} \cdot \left(\sum_{r=0}^l \psi_{l-r} \cdot \frac{S_1^r}{r!}\right) + \delta_{s+1,0}\\
    &= \frac{S_1^{s+1}}{(s+1)!}+ \sum_{l=1}^s \sum_{r=0}^l \delta_{s+1-l,0}\cdot \psi_{l-r} \cdot \frac{S_1^r}{r!} + \delta_{s+1,0}\\
    &= \frac{S_1^{s+1}}{(s+1)!}+ \sum_{l=0}^s \sum_{r=0}^l \delta_{s+1-l,0}\cdot \psi_{l-r} \cdot \frac{S_1^r}{r!}\\
    &= \frac{S_1^{s+1}}{(s+1)!}+ \sum_{r=0}^s\sum_{l=r}^s  \delta_{s+1-l,0}\cdot \psi_{l-r} \cdot \frac{S_1^r}{r!}\\
    &= \frac{S_1^{s+1}}{(s+1)!}+ \sum_{r=0}^s\sum_{l'=0}^{s-r}  \delta_{s+1-r-l',0}\cdot \psi_{l'} \cdot \frac{S_1^r}{r!}\\
    &= \frac{S_1^{s+1}}{(s+1)!}+ \sum_{r=0}^s\psi_{s+1-r} \cdot \frac{S_1^r}{r!}\\
    &=  \sum_{r=0}^{s+1}\psi_{s+1-r} \cdot \frac{S_1^r}{r!}
\end{align}
completing the proof by induction.
\end{proof}
\noindent\rule{\textwidth}{1pt}
\begin{remark}
Hence we have shown that for all $s\geq 1$ 
\begin{align}\omega(s) = \sum_{r=0}^{s} \psi_{s-r}\frac{S_1^r}{r!}= \frac{S_1^s}{s!}+ \sum_{r=0}^{s-2} \psi_{s-r}\frac{S_1^r}{r!}\end{align}
or in other words 
\begin{align}\omega(s) = \frac{(\ln{\beta_{t-1,t+k-1}})^s}{s!}+ \sum_{r=0}^{s-2} \psi_{s-r}\frac{(\ln{\beta_{t-1,t+k-1}})^r}{r!} \sim \frac{(\ln{T})^s}{s!}+ \sum_{r=0}^{s-2} \psi_{s-r}\frac{(\ln{T})^r}{r!}\end{align}
as $t+k=T \rightarrow \infty$, which is roughly the polynomial in $\ln{T}$ of degree $s$ we were anticipating. 
\end{remark}
\noindent\rule{\textwidth}{1pt}
\subsubsection{Estimating $\theta$} \label{subsub_theta}
\begin{remark}
Let us now recall the definition for all $s\geq 1$, 
\begin{align}\theta_{l:k}^{(t)}(s) &\myeq \sum_{l\leq i_1<\ldots<i_s<k} \left(\frac{1}{t+i_s}+1_{i_s-i_{s-1}=1}\right)\cdot \left(\frac{1}{t+i_{s-1}}+1_{i_{s-1}-i_{s-2}=1}\right)\cdot \ldots \\
&\ldots\cdot \left(\frac{1}{t+i_2}+1_{i_2-i_1=1}\right)\cdot  \left(\frac{1}{t+i_1}+1_{i_1=0}\right)\end{align}
which we would like to estimate using $\omega_{l:k}^{(t)}(s)$.

In order to build a first intuition, let us look at how it plays out for small values for $s$. 

\paragraph{Notation.} In this subsection we omit the superscript $(t)$ notation because there is no ambiguity. We will also occasionally do the abuse of notation and assume $\omega_{l:k}(0)=1$ for all $l<k$.

For $\underline{s=1}$, we get 
\begin{align}\theta_{0:k}(1) = 1 + \omega_{0:k}(1)\end{align}

For $\underline{s=2}$, we get 
\begin{align}\theta_{0:k}(2) = 1 + \omega_{1:k}(1)+ \omega_{0:k-1}(1) + \omega_{0:k}(2)\end{align}

In what follows, we will use the following recursive formula quite frequently
\begin{align} \label{recc_theta}
    \theta_{0:k}(s+1) &= \theta_{0:k-1}(s)+ \sum_{j=s}^{k-1} \frac{1}{t+j} \theta_{0:j}(s)
\end{align}
Hence for $\underline{s=3}$, we get 
\begin{align}\theta_{0:k}(3) &= 1 + \omega_{1:k-1}(1)+ \omega_{0:k-2}(1) +
\omega_{2:k}(1) +\omega_{0:k-1}(2)\\
&+ \omega_{1:k}(2) + \sum_{j=2}^{k-1} \frac{\omega_{0:j-1}(1)}{t+j} + \omega_{0:k}(3)\end{align}
Now let us further observe that for all $s \geq 1$ and $0\leq r \leq l$, we have 
\begin{align} \label{ineq_omega}
    \omega_{l+r:k+r}(s) \leq \omega_{l:k}(s) \leq \omega_{l-r:k-r}(s)
\end{align}
This implies that 
\begin{align}1+2\cdot \omega_{1:k}(1)+ \omega_{0:k}(2) \leq \theta_{0:k}(2)\leq 1+2\cdot \omega_{0:k-1}(1)+ \omega_{0:k}(2)\end{align}
and, similarly,
\begin{align}1+3 \cdot \omega_{2:k}(1)+3\cdot \omega_{1:k}(2)+ \omega_{0:k}(3) \leq \theta_{0:k}(3) \leq 1+3 \cdot \omega_{0:k-2}(1)+3\cdot \omega_{0:k-1}(2)+ \omega_{0:k}(3)\end{align}
\paragraph{Hypothesis.} We can thus see the binomial coefficients arising, and we would expect that in general, we have 
\begin{align}\sum_{r=0}^s {s \choose r} \cdot  \omega_{0:k-s+r}(r) \geq \theta_{0:k}(s)\geq \sum_{r=0}^s {s \choose r}\cdot \omega_{s-r:k}(r)\end{align}
\end{remark}
\noindent\rule{\textwidth}{1pt}
\begin{lemma} \label{ineq_theta}
For all $k \geq s \geq 1$, we have 
\begin{align}\sum_{r=0}^s {s \choose r} \cdot  \omega_{0:k-s+r}^{(t)}(r) \geq \theta_{0:k}^{(t)}(s)\geq \sum_{r=0}^s {s \choose r}\cdot \omega_{s-r:k}^{(t)}(r)\end{align}
\end{lemma}
\begin{proof}
Let us prove this lemma by induction on $s$. The cases $s=1,2,3$ have already been treated in the previous remark.

Let us now assume that the claim holds for $s$, and prove it for $s+1$ using the recursive formula \begin{align}\theta_{0:k}(s+1) = \theta_{0:k-1}(s)+ \sum_{j=s}^{k-1} \frac{1}{t+j} \theta_{0:j}(s)\end{align}
For the lower bound, using the induction hypothesis, we get
\begin{align}
    \theta_{0:k}(s+1) &\geq  \sum_{r=0}^s {s \choose r} \cdot \omega_{s-r:k-1}(r) + \sum_{j=s}^{k-1} \frac{1}{t+j} \sum_{r=0}^s {s \choose r} \cdot \omega_{s-r:j}(r)\\
    &=  \sum_{r=0}^s {s \choose r} \cdot \omega_{s-r:k-1}(r) + \sum_{r=0}^s {s \choose r}\cdot \sum_{j=s}^{k-1} \frac{1}{t+j}  \cdot \omega_{s-r:j}(r)\\
    &=\sum_{r=0}^s {s \choose r} \cdot \omega_{s-r:k-1}(r) + \sum_{r=0}^s {s \choose r}\cdot \omega_{s-r:k}(r+1)\\
    &=\sum_{r=0}^s {s \choose r} \cdot \omega_{s-r:k-1}(r) + \sum_{r=1}^{s+1} {s \choose r-1}\cdot \omega_{s-r+1:k}(r)\\
    &= 1+\omega_{0:k}(s+1) + \sum_{r=1}^s \left[{s \choose r}+{s \choose r-1}\right] \cdot \omega_{s-r+1:k}(r)\\
    &= 1+\omega_{0:k}(s+1) + \sum_{r=1}^s {s+1 \choose r} \cdot \omega_{s-r+1:k}(r)\\
    &= \sum_{r=0}^{s+1} {s+1 \choose r} \cdot \omega_{s-r+1:k}(r)\\
\end{align}
For the upper bound, using the induction hypothesis, we get
\begin{align}
    \theta_{0:k}(s+1) &\leq  \sum_{r=0}^s {s \choose r} \cdot \omega_{0:k-1-(s-r)}(r) + \sum_{j=s}^{k-1} \frac{1}{t+j} \sum_{r=0}^s {s \choose r} \cdot \omega_{0:j-(s-r)}(r)\\
    &=  \sum_{r=0}^s {s \choose r} \cdot \omega_{0:k-1-(s-r)}(r) + \sum_{r=0}^s {s \choose r}\cdot \sum_{j=s}^{k-1} \frac{1}{t+j}  \cdot \omega_{0:j-(s-r)}(r)\\
    &\leq  \sum_{r=0}^s {s \choose r} \cdot \omega_{0:k-1-(s-r)}(r) + \sum_{r=0}^s {s \choose r}\cdot \sum_{j=s}^{k-1} \frac{1}{t+j-(s-r)}  \cdot \omega_{0:j-(s-r)}(r)\\
    &=  \sum_{r=0}^s {s \choose r} \cdot \omega_{0:k-1-(s-r)}(r) + \sum_{r=0}^s {s \choose r}\cdot \sum_{j'=r}^{k-1-(s-r)} \frac{1}{t+j'}  \cdot \omega_{0:j'}(r)\\
    &=\sum_{r=0}^s {s \choose r} \cdot \omega_{0:k-1-(s-r)}(r) + \sum_{r=0}^s {s \choose r}\cdot \omega_{0:k-(s-r)}(r+1)\\
    &=\sum_{r=0}^s {s \choose r} \cdot \omega_{0:k-1-(s-r)}(r) + \sum_{r=1}^{s+1} {s \choose r-1}\cdot \omega_{0:k-(s+1-r)}(r)\\
    &= 1+\omega_{0:k}(s+1) + \sum_{r=1}^s \left[{s \choose r}+{s \choose r-1}\right] \cdot\omega_{0:k-(s+1-r)}(r)\\
    &= 1+\omega_{0:k}(s+1) + \sum_{r=1}^s {s+1 \choose r} \cdot \omega_{0:k-(s+1-r)}(r)\\
    &= \sum_{r=0}^{s+1} {s+1 \choose r} \cdot \omega_{0:k-(s+1-r)}(r)\\
\end{align}    
completing the proof by induction.
\end{proof}
\noindent\rule{\textwidth}{1pt}
\begin{remark}
Let us recall that 
\begin{align}\omega_{l:k}(r) = \sum_{q=0}^r \psi_{r-q} \frac{(\ln{\beta_{t+l-1,t+k-1}})^q}{q!}\end{align}
Thus the difference between the upper-bound and the lower-bound becomes 
\begin{align}
    \sum_{r=0}^s {s \choose r} \left[\omega_{o:k-(s-r)}(r)-\omega_{s-r:k}(r)\right] = \sum_{r=0}^s {s \choose r} \cdot \left[\sum_{q=0}^r \psi_{r-q} \frac{(\ln{\beta_{t-1,t+k-(s-r)-1}})^q- (\ln{\beta_{t+s-r-1,t+k-1}})^q}{q!}\right]  
\end{align}
which converges to zero as $T=t+k\rightarrow \infty$.
\end{remark}

\noindent\rule{\textwidth}{1pt}
\subsubsection{Putting it all together} \label{subsub_all}
\begin{remark}
Now it is time to turn to $\chi_{0:k}^{(t)}(s)$ and finally put it all together, so that we can finally estimate 
\begin{align}\frac{d s_{t+k}}{d h_t} = \sum_{s=0}^k V^s \cdot \chi_{0:k}^{(t)}(s)\end{align}
and get the asymptotic estimate when $T=t+k \rightarrow \infty$.

Let us recall that \begin{align}\chi_{0:k}^{(t)}(s) &= \frac{1}{t+k}\cdot \theta_{0:k}^{(t)}(s)+\frac{1}{t+k-1}\cdot\theta_{0:k-1}^{(t)}(s-1)+\ldots\\
&\ldots +\frac{1}{t+k-s+1}\cdot \theta_{0:k-s+1}^{(t)}(1) + \frac{1}{t+k-s}+1_{k=s} \end{align}
Using the abuse of notation $\theta_{l:k}(0)=1$ for $l<k$, we can rewrite it as follows
\begin{align}\chi_{0:k}^{(t)}(s) = 1_{k=s}+ \sum_{i=0}^s \frac{1}{t+k-i}\cdot \theta_{0:k-i}(s-i)\end{align}
The idea is to use the inequality from lemma \ref{ineq_theta}, and get a similar result for $\chi_{0:k}^{(t)}(s)$, then show that the lower and upper bound are no more than $\Theta(1/T)$ apart, thus enabling us to eventually get an asymptotic estimate for $\frac{d s_{t+k}}{d h_t}$.

We are also omitting the superscript $(t)$ notation here because of lack of ambiguity.
\end{remark}
\noindent\rule{\textwidth}{1pt}
\begin{lemma}\label{ineq_chi} For all $s\geq 0$ and $k\geq 1$, we have
\begin{align}1_{k=s}+\frac{1}{t+k} \cdot \sum_{r=0}^s {s+1 \choose r+1} \cdot \omega_{s-r:k}(r)\leq \chi_{0:k}(s)\leq 1_{k=s}+\frac{1}{t+k-s} \cdot \sum_{r=0}^s {s+1 \choose r+1} \cdot \omega_{0:k-(s-r)}(r)\end{align}
\end{lemma}
\begin{proof}
Using the upper-bound of lemma \ref{ineq_theta}, we get 
\begin{align}
    \chi_{0:k}(s) &= 1_{k=s}+ \sum_{i=0}^s \frac{1}{t+k-i}\cdot \theta_{0:k-i}(s-i)\\
    &\leq 1_{k=s}+ \sum_{i=0}^s \frac{1}{t+k-i}\cdot \sum_{r=0}^{s-i} {s-i \choose r} \cdot \omega_{0:k-s+r}(r)\\
    &\leq 1_{k=s}+\frac{1}{t+k-s} \cdot\sum_{i=0}^s \sum_{r=0}^{s-i} {s-i \choose r} \cdot \omega_{0:k-s+r}(r)\\
    &= 1_{k=s}+\frac{1}{t+k-s}\cdot  \sum_{r=0}^{s}\left[\sum_{i=0}^{s-r} {s-i \choose r}\right] \cdot \omega_{0:k-s+r}(r)\\
    &= 1_{k=s}+\frac{1}{t+k-s}\cdot  \sum_{r=0}^{s} {s+1 \choose r+1}\cdot \omega_{0:k-s+r}(r)
\end{align}
Similarly, using the lower-bound of lemma \ref{ineq_theta}, we get \begin{align}
    \chi_{0:k}(s) &= 1_{k=s}+ \sum_{i=0}^s \frac{1}{t+k-i}\cdot \theta_{0:k-i}(s-i)\\
    &\geq 1_{k=s}+ \sum_{i=0}^s \frac{1}{t+k-i}\cdot \sum_{r=0}^{s-i} {s-i \choose r} \cdot \omega_{s-r:k}(r)\\
    &\geq 1_{k=s}+\frac{1}{t+k} \cdot\sum_{i=0}^s \sum_{r=0}^{s-i} {s-i \choose r} \cdot \omega_{s-r:k}(r)\\
    &= 1_{k=s}+\frac{1}{t+k}\cdot  \sum_{r=0}^{s}\left[\sum_{i=0}^{s-r} {s-i \choose r}\right] \cdot \omega_{s-r:k}(r)\\
    &= 1_{k=s}+\frac{1}{t+k}\cdot  \sum_{r=0}^{s} {s+1 \choose r+1}\cdot \omega_{s-r:k}(r)
\end{align}
\end{proof}
\noindent\rule{\textwidth}{1pt}
\begin{lemma} \label{asymptotic_chi}
For all $s\geq 0$, we have 
\begin{align}\chi_{0:k}(s )=1_{k=s}+\frac{1}{t+k}\left[\sum_{r=0}^s{s+1 \choose r+1} \cdot \omega_{s-r:k}(r)\right]+\Theta \left(\frac{1}{t+k}\right)\end{align}
for all large enough $k>1$, and where the implicit constants from the $\Theta(.)$ notation are dependent on $s$.
\end{lemma}
\begin{proof}
Building on the previous lemma \ref{ineq_chi}, and substracting the lower bound from the upper bound, we get
\begin{align}
    \sum_{r=0}^s {s+1 \choose r+1} \cdot \left[\frac{\omega_{0:k-s+r}(r)}{t+k-s}-\frac{\omega_{s-r:k}(r)}{t+k}\right] &= \sum_{r=0}^s\sum_{q=0}^r {s+1 \choose r+1} \frac{\psi_{r-q}}{q!} \cdot \left[\frac{(\ln{\beta_{t-1,t+k-s+r-1}})^q}{t+k-s}-\frac{(\ln{\beta_{t+s-r-1,t+k-1}})^q}{t+k}\right]
\end{align}
When assuming that for large $k$, we have  \begin{align}(\ln{\beta_{t-1,t+k-s+r-1}})^q \approx (\ln{\beta_{t+s-r-1,t+k-1}})^q\end{align}
then 
\begin{align}
    \frac{(\ln{\beta_{t-1,t+k-s+r-1}})^q}{t+k-s}-\frac{(\ln{\beta_{t+s-r-1,t+k-1}})^q}{t+k} &\approx \frac{1}{t+k}\cdot \left[\frac{s}{t+k-s} \cdot (\ln{\beta_{t-1,t+k-s+r-1}})^q\right]\\
    &\leq \frac{1}{t+k}\cdot \left[\frac{s}{t+k-s} \cdot (\ln{\beta_{t-1,t+k-s+r-1}})^s\right]\\
    &\leq \frac{\tau_s}{t+k}
\end{align}
for some $\tau_s > 0$ depending on $s$, for all sufficiently large $k$. 

In other words, we have 
\begin{align}\sum_{r=0}^s {s+1 \choose r+1} \cdot \left[\frac{\omega_{0:k-s+r}(r)}{t+k-s}-\frac{\omega_{s-r:k}(r)}{t+k}\right] \leq \frac{\tilde{\tau}_s}{t+k}\end{align}
for for some $\tilde{\tau}_s > 0$ depending on $s$, for all sufficiently large $k$.

Meanwhile, for all large enough $k$, we have
\begin{align}
    \sum_{r=0}^s {s+1 \choose r+1} \cdot \left[\frac{\omega_{0:k-s+r}(r)}{t+k-s}-\frac{\omega_{s-r:k}(r)}{t+k}\right] &\approx \frac{s}{(t+k)(t+k-s)}\cdot  \sum_{r=0}^s\sum_{q=0}^r {s+1 \choose r+1} \frac{\psi_{r-q}}{q!} \cdot (\ln{\beta_{t-1,t+k-s+r-1}})^q\\
    &\geq \frac{\tau_s'}{(t+k)^2}\cdot  \sum_{r=0}^s\sum_{q=0}^r {s+1 \choose r+1} \frac{\psi_{r-q}}{q!} \cdot (\ln{\beta_{t-1,t+k-s+r-1}})^q\\
    &\geq  \frac{\tau_s'}{(t+k)^2}\cdot  \sum_{r=0}^s\sum_{q=0}^r  \frac{\psi_{r-q}}{q!} \cdot (\ln{\beta_{t-1,t+k-s+r-1}})^q\\
    &=  \frac{\tau_s'}{(t+k)^2}\cdot  \sum_{q=0}^s \sum_{r'=0}^{s-q}  \frac{\psi_{r'}}{q!} \cdot (\ln{\beta_{t-1,t+k-s+r'+q-1}})^q\\
    &\approx \frac{\tau_s'}{(t+k)^2}\cdot  \sum_{q=0}^s \left( \sum_{r'=0}^{s-q} \psi_{r'}\right)\cdot  \frac{(\ln{(t+k)})^q}{q!}\\
    &\geq \frac{\tau_s''}{(t+k)^2}\cdot  \sum_{q=0}^s  \frac{(\ln{(t+k)})^q}{q!}\\
    &\approx \frac{\tau_s'' \cdot \exp{[\ln{(t+k)}]}}{(t+k)^2}\\
    &= \frac{\tau_s''}{t+k}
\end{align}
for some $\tau'_s, \tau''_s> 0$ depending on $s$.
\end{proof}
\noindent\rule{\textwidth}{1pt}
\begin{proposition} \label{magical_proposition}
If $V$ is a normal matrix with eigenvalues $\lambda_1,\lambda_2,\ldots,\lambda_n$ of modulus smaller than $1$, then 
\begin{align}\frac{d s_{T}}{d h_t} = P \Lambda_T P^*\end{align}
where $P^*$ is the conjugate transpose of the unitary matrix $P$ (independent of $T$) and where $\Lambda_T$ is a diagonal matrix satisfying 
\begin{align}\left(\Lambda_T\right)_{ii} \sim T^{-1}\cdot c + T^{\lambda_i-1}\cdot c'\end{align} for some positive real constants $c,c'$, as $T\rightarrow \infty$.
\end{proposition}
\begin{proof}
Let $V= P\Lambda P^*$ be the Schur decomposition of $V$, with $\Lambda = \textrm{diag}(\lambda_1,\lambda_2,\ldots,\lambda_n)$. Note that since we supposed that $V$ is normal, we thus have that the Schur matrix $\Lambda$ is indeed diagonal and is composed of the eigenvalues on the diagonal. 

Based on lemma \ref{asymptotic_chi}, one can show that there exists a function $g:\mathbb{N}\rightarrow \mathbb{R}_0^{+}$ such that 
\begin{align}\chi_{0:k}(s) =1_{k=s}+\frac{1}{t+k}\left[\sum_{r=0}^s{s+1 \choose r+1} \cdot \omega_{s-r:k}(r)+g(s)\right]\end{align}
Thus 
\begin{align}
    \frac{d s_{t+k}}{d h_t} &= \sum_{s=0}^k V^s \cdot \chi_{0:k}(s)\\
    &= V^k + \frac{1}{t+k}\left[\sum_{s=0}^k g(s)\cdot V^s + \sum_{s=0}^k \sum_{r=0}^s {s+1 \choose r+1} \cdot \omega_{s-r:k}(r)\cdot V^s\right]\\
    &= V^k + \frac{1}{t+k}\left[\sum_{s=0}^k g(s)\cdot V^s + \sum_{s=0}^k \sum_{r=0}^s\sum_{q=0}^r {s+1 \choose r+1} \cdot \psi_{r-q} \frac{(\ln{\beta_{t+s-r-1,t+k-1}})^q}{q!}\cdot V^s\right]\\
\end{align}
Since the eigenvalues of $V$ are of modulus smaller than $1$, we can assume that there exists a constant $d>0$ (dependent on the choice of eigenvalues of $V$) such that for all $k>d$ we have $V^k \approx 0$.

Furthermore since $V^m = (P\Lambda P^*)^m = P \Lambda^m P^*$ for all $m\in \mathbb{N}_0$, while keeping in mind that we pick $T=t+k$, we can write 
\begin{align}
    \Lambda_T &= \frac{1}{T}\left[\sum_{s=0}^d g(s)\cdot \Lambda^s + \sum_{s=0}^d \sum_{r=0}^s\sum_{q=0}^r {s+1 \choose r+1} \cdot \psi_{r-q} \frac{(\ln{\beta_{t+s-r-1,T-1}})^q}{q!}\cdot \Lambda^s\right]\\
    &= \frac{1}{T}\left[\sum_{s=0}^d g(s)\cdot \Lambda^s + \sum_{s=0}^d \sum_{q=0}^s  \sum_{r=q}^s {s+1 \choose r+1} \cdot \psi_{r-q} \frac{(\ln{\beta_{t+s-r-1,T-1}})^q}{q!}\cdot \Lambda^s\right]\\
    &= \frac{1}{T}\left[\sum_{s=0}^d g(s)\cdot \Lambda^s + \sum_{q=0}^d \sum_{s=q}^d  \sum_{r=q}^s {s+1 \choose r+1} \cdot \psi_{r-q} \frac{(\ln{\beta_{t+s-r-1,T-1}})^q}{q!}\cdot \Lambda^s\right]\\
    &= \frac{1}{T}\left[\sum_{s=0}^d g(s)\cdot \Lambda^s + \sum_{q=0}^d \sum_{s=q}^d  \sum_{r=q}^s {s+1 \choose r+1} \cdot \psi_{r-q} \frac{(\Lambda \cdot \ln{\beta_{t+s-r-1,T-1}})^q}{q!}\cdot \Lambda^{s-q}\right]\\
    &= \frac{1}{T}\left[\sum_{s=0}^d g(s)\cdot \Lambda^s + \sum_{q=0}^d \sum_{s'=0}^{d-q}  \sum_{r'=0}^{s'} {s'+q+1 \choose r'+q+1} \cdot \psi_{r'} \frac{(\Lambda \cdot \ln{\beta_{t+s'-r'-1,T-1}})^q}{q!}\cdot \Lambda^{s'}\right]\\
    &\sim \frac{1}{T}\left[\sum_{s=0}^d g(s)\cdot \Lambda^s + \sum_{q=0}^d \sum_{s'=0}^{d-q}  \sum_{r'=0}^{s'} {s'+q+1 \choose r'+q+1} \cdot \psi_{r'} \frac{(\Lambda \cdot \ln{T})^q}{q!}\cdot \Lambda^{s'}\right]\\
    &= \frac{1}{T}\left[\sum_{s=0}^d g(s)\cdot \Lambda^s\right] +\frac{1}{T}\left[ \sum_{q=0}^d \frac{(\Lambda \cdot \ln{T})^q}{q!}\cdot \left( \sum_{s'=0}^{d-q}  \sum_{r'=0}^{s'} {s'+q+1 \choose r'+q+1} \cdot \psi_{r'} \cdot \Lambda^{s'}\right)\right]\\
    &\approx \frac{1}{T}\left[\sum_{s=0}^d g(s)\cdot \Lambda^s\right] +\frac{1}{T} \exp{(\Lambda \cdot \ln{T})} \cdot (c_0+c_1 \cdot \Lambda + \ldots + c_d \cdot \Lambda^d)\\
    &\sim \frac{c}{T} +\frac{c'}{T} \exp{(\Lambda \cdot \ln{T})}
\end{align}
for some positive constants $c',c,c_0,c_1,\ldots,c_d$.

Hence 
\begin{align}(\Lambda_T)_{ii} \sim c\cdot T^{-1} + c'\cdot T^{\lambda_i-1}\end{align}
\end{proof}
\noindent\rule{\textwidth}{1pt}
\begin{theorem} \label{main_theorem_unif}
If $V$ is a normal matrix with eigenvalues of modulus smaller than $1$, then 
\begin{align}\|\frac{d s_{T}}{d h_t}\| =\Omega(1/T)\end{align}
as $T\rightarrow \infty$. (here $\|.\|$ is the Frobenius norm.)
\end{theorem}
\begin{proof}
Let us start off with the observation that 
\begin{align}T^{-1}\cdot c + T^{\lambda_i-1}\cdot c' = \Omega\left(T^{-\min{(1,1-\mathfrak{Re}(\lambda_i))}}\right)\end{align}
as $T \rightarrow \infty$. And thus, by using proposition \ref{magical_proposition}, we get \begin{align}\|\frac{d s_{T}}{d h_t}\| =\Omega(T^{-\eta})\end{align}
where 
\begin{align}\eta = \min_{i=1,\ldots,n}{\{\min{(1,1-\mathfrak{Re}(\lambda_i))}\}} \leq 1\end{align}
\end{proof}

\noindent\rule{\textwidth}{1pt} 
\begin{remark} \label{normality_remark}
Note that $V$ being normal is not a necessary condition for the generality of the theorem to hold. We simply chose $V$ to be normal in order to make the calculations less cumbersome.

In case $V$ is non-normal, its Schur matrix $\Lambda$ becomes triangular instead of diagonal. In fact, if $t_{i,j}$ are the off-diagonal elements of Schur matrix of $V$ (with $i<j$), then 
\begin{align}\|V\| = \sqrt{\textrm{Tr}(V^*V)} = \sqrt{\textrm{Tr}(\Lambda^* \Lambda)} = \sqrt{\sum_{i=1}^n |\lambda_i|^2+ \sum_{i<j}|t_{i,j}|^2}\geq \sqrt{\sum_{i=1}^n |\lambda_i|^2}\end{align}
Thus every lower bound on $\sqrt{\sum_{i=1}^n |\lambda_i|^2}$ induces a lower bound on $\|V\|$, and in particular an asymptotic lower bound on the modulus of one of the eigenvalues of $\frac{d s_{T}}{d h_t}$ induces an asymptotic lower bound on $\|\frac{d s_{T}}{d h_t}\|$. 
\end{remark}
\noindent\rule{\textwidth}{1pt}
\subsection{Sparse relevance case with bounded dependency depth}\label{sparse_appendix}
\begin{remark}
Similarly to remark \ref{gen_assumptions}, we are going to assume for this subsection:
\begin{itemize}
    \item no non-linearity in the hidden-to-hidden connection: $J_t = V$ for all $t$.
    \item all assumptions from Remark \ref{remark_assumption}.
    \item $\kappa$-sparse attention: for each $t\geq 1$, there are at most $\kappa\leq t$ values for $i$ such that $\alpha_{i,t} \neq 0$. (Let us define $\kappa_t \myeq \left|\{i \textrm{ such that } \alpha_{i,t} \neq 0\}\right|$)
    \item uniform attention across attended states: for all $t\geq 1$, and all $i\leq t$ such that $\alpha_{i,t}\neq 0$, we have  $\alpha_{i,t}=1/\kappa_t \geq 1/\kappa$.
\end{itemize}
\end{remark}
\noindent\rule{\textwidth}{1pt}
\begin{remark}\label{sparse_assumptions} Similarly to remark \ref{unif_assumptions}, let us recall that 
\begin{align}X_{i,t} = \left(h_i-\sum_{j=1}^t \alpha_{j,t} h_j\right) \cdot \frac{\partial e_{i,t}}{\partial h_i}\end{align}

and that 

\begin{align}
    \sum_{i=1}^t \alpha_{i,t} Y_{i,t} &= \sum_{i=1}^t \alpha_{i,t} \cdot h_i \cdot \left( \frac{\partial e_{i,t}}{\partial s_{t-1}}-\sum_{j=1}^t \alpha_{j,t} \cdot \frac{\partial e_{j,t}}{\partial s_{t-1}}\right)\\
    &= \sum_{i=1}^t \alpha_{i,t} \cdot h_i \cdot \frac{\partial e_{i,t}}{\partial s_{t-1}} - \sum_{i=1}^t \alpha_{i,t} \left(\sum_{j=1}^t \alpha_{j,t}\cdot h_j\right) \cdot \frac{\partial e_{i,t}}{\partial s_{t-1}}\\
    &= \sum_{i=1}^t \alpha_{i,t}\cdot \left(h_i- \sum_{j=1}^t \alpha_{j,t} h_j\right)\cdot \frac{\partial e_{i,t}}{\partial s_{t-1}}
\end{align}
Thus we can see that both expressions have the common factor $\left(h_i- \sum_{j=1}^t \alpha_{j,t} h_j\right)$.

By defining 
\begin{align}A_t \myeq \{i \textrm{ such that } \alpha_{i,t}\neq 0\}\end{align}
we see that 
\begin{align}h_i- \sum_{j=1}^t \alpha_{j,t} h_j = h_i - \frac{1}{\kappa_t} \sum_{j \in A_t} h_j\end{align}

and we are going to assume for the sake of simplicity that 
\begin{align}h_i \approx \frac{1}{\kappa_t} \sum_{j \in A_t} h_j\end{align}

and thus $X_{i,t} \approx 0$ and $\sum_{i=1}^t \alpha_{i,t} Y_{i,t} \approx 0$. 

Recalling the expression from corollary \ref{main_corollary} and that $f(h_t,c_t)= h_t+c_t$ by remark \ref{remark_assumption}, and that $J_t = V$ for all $t$, this will give for all $k'\geq 0$
\begin{align}
    E_{k'}^{(t)} &= \left(\frac{1_{t \in A_{t+k'}}}{\kappa_{t+k'}}+1_{k'=0}\right)\cdot \textrm{I}
\end{align}
and for all $k\geq j$, we get
\begin{align}
 F_{k+1,j}^{(t)} &= \left(\frac{1_{t+j+1\in A_{t+k+1}}}{\kappa_{t+k+1}} + 1_{k=j}\right)\cdot V   
\end{align}
Hence by recalling proposition \ref{main_prop}, the main expression of interest becomes 
\begin{align}\frac{d s_{t+k}}{d h_t} = \sum_{s=0}^k \bar{\xi}_{0:k}^{(t)}(s) = \sum_{s=0}^k V^s \cdot \chi_{0:k}^{(t)}(s)\end{align}

where 
\begin{align}\chi_{0:k}^{(t)}(s)&\myeq \sum_{0\leq i_1<\ldots <i_s<k}\left(\frac{1_{t+i_s+1 \in A_{t+k}}}{\kappa_{t+k}}+ 1_{k-i_{s}=1}\right)\cdot\left(\frac{1_{t+i_{s-1}+1 \in A_{t+i_s}}}{\kappa_{t+i_s}}+ 1_{i_{s}-i_{s-1}=1}\right)\cdot\\
&\ldots \cdot \left(\frac{1_{t+i_{1}+1 \in A_{t+i_2}}}{\kappa_{t+i_2}}+ 1_{i_{2}-i_{1}=1}\right) \cdot \left(\frac{1_{t \in A_{t+i_1}}}{\kappa_{t+i_1}}+1_{i_1=0}\right)\end{align}

\end{remark}
\noindent\rule{\textwidth}{1pt}
\begin{remark}
Let us now have a look at how we could potentially simplify the analysis of $\chi_{0:k}^{(t)}(s)$. 

If we further assume $V$ to be normal we can write 

\begin{align}V = P \Lambda P^*\end{align}
where $\Lambda= \textrm{diag}(\lambda_1,\lambda_2,\ldots, \lambda_n)$ is the diagonal matrix consisting of the eigenvalues of $V$, and $P^*$ is the conjugate transpose of $P$.

Hence, we can rewrite

\begin{align}\frac{d s_{t+k}}{d h_t} = \sum_{s=0}^k V^s \cdot \chi_{0:k}^{(t)}(s) = P\cdot \left(\sum_{s=0}^k \Lambda^s \cdot \chi_{0:k}^{(t)}(s)\right)\cdot P^*\end{align}

We can therefore see that the asymptotic behaviour of $\frac{d s_{t+k}}{d h_t}$ depends largely on the asymptotic behaviour of the modulus of the complex-valued polynomial 

\begin{align}p_{0:k}(\lambda) \myeq \sum_{s=0}^k \lambda^s \cdot \chi_{0:k}^{(t)}(s)\end{align}
and thus 

\begin{align}\|\frac{d s_{t+k}}{d h_t}\| = \sqrt{\sum_{i=1}^n|p_{0:k}(\lambda_i)|^2}\end{align}
where $\|.\|$ is the Frobenius norm. Hence in order to prove that \begin{align}\|\frac{d s_{t+k}}{d h_t}\| = \Omega(1/\kappa^d)\end{align} for all large enough $k$ (note that $k$ and $\kappa$ here are two different symbols), it would suffice to show that there exists $\lambda \in \{\lambda_1,\ldots,\lambda_n\}$ such that, for all large enough $k$, we have  
\begin{align}|p_{0:k}(\lambda)| = \Omega(1/\kappa^d)\end{align}

For simplicity we are going to assume that for all $t$, we have $\kappa_t= \kappa$.

Let us further define for all $s\geq 1$, 
\begin{align}f_{0:k}^{(s)}(i_1,\ldots,i_s) &\myeq \left(\frac{1_{t+i_s+1 \in A_{t+k}}}{\kappa_{t+k}}+ 1_{k-i_{s}=1}\right)\cdot\left(\frac{1_{t+i_{s-1}+1 \in A_{t+i_s}}}{\kappa_{t+i_s}}+ 1_{i_{s}-i_{s-1}=1}\right)\cdot \ldots \\
&\ldots \cdot \left(\frac{1_{t+i_{1}+1 \in A_{t+i_2}}}{\kappa_{t+i_2}}+ 1_{i_{2}-i_{1}=1}\right)\cdot \left(\frac{1_{t \in A_{t+i_1}}}{\kappa_{t+i_1}}+1_{i_1=0}\right) \end{align}
whenever $(i_1,\ldots,i_n)$ satisfies $0\leq i_1<i_2<\ldots<i_s<k$, and 
\begin{align}f_{0:k}^{(s)}(i_1,\ldots,i_s) \myeq 0\end{align}
otherwise. \\
% By the hypothesis on the dependency depth $d$, we know that for each $k$, there exists $s'\leq d$ and $(i_1,i_2,\ldots,i_{s'})$ such that 
% $$f_{0:k}^{(s')}(i_1,\ldots,i_{s'}) \geq \left(\frac{1}{\kappa}\right)^{s'+1} \geq  \left(\frac{1}{\kappa}\right)^{d+1}$$
% Hence if $\lambda$ is real and positive, then for all large enough $k$, we have  
% $$|p_{0:k}(\lambda)| = \Omega(1/\kappa^d)$$

% In other words, if $V$ has at least one positive real eigenvalue then $$\|\frac{d s_{t+k}}{d h_t}\| = \Omega(1/\kappa^d)$$ for all large enough $k$.

% we assume that $V$ has at least one positive and real eigenvalue 
% \textbf{Definition.} Let us further define the set $\Gamma_{0:k}^{(s)}$ to be the set of all $s$-tuples $(i_1,i_2,\ldots,i_s)$ such that
% \begin{itemize}
%     \item $0 \leq i_1 < i_2 < \ldots < i_s <k$
%     \item $t+i_s+1 \in A_{t+k}$
%     \item $t+i_{j-1}+1 \in A_{t+i_j}$ for all $2\leq j \leq s$.
%     \item $t \in A_{t+i_1}$
% \end{itemize}
% We will further assume that for all $t'\geq 1$, and all $i \in A_{t'}$, we have $\alpha_{i,t'} = 1/\kappa$. 

% Thus we have 
% $$\chi_{0:k}^{(t)}(s)\myeq \sum_{(i_1,i_2,\ldots,i_s)\in \Gamma_{0:k}^{(s)}}\left(\frac{1}{\kappa}+ 1_{k-i_{s}=1}\right)\cdot\left(\frac{1}{\kappa}+ 1_{i_{s}-i_{s-1}=1}\right)\cdot \ldots \cdot \left(\frac{1}{\kappa}+1_{i_1=0}\right)$$
\end{remark}
\noindent\rule{\textwidth}{1pt}
\begin{theorem}\label{main_theorem_sparse}
Given the $\kappa$-sparsity assumption and the dependency depth $d$, we have that if $V$ is normal and has one positive real eigenvalue, then 

\begin{align}\|\frac{d s_{t+k}}{d h_t}\| = \Omega(1/\kappa^d)\end{align} for all large enough $k$.

\end{theorem}
\begin{proof}
By the hypothesis on the dependency depth $d$, we know that for each $k$, there exists $s'\leq d$ and $(i_1,i_2,\ldots,i_{s'})$ such that 
\begin{align}f_{0:k}^{(s')}(i_1,\ldots,i_{s'}) \geq \left(\frac{1}{\kappa}\right)^{s'+1} \geq  \left(\frac{1}{\kappa}\right)^{d+1}\end{align}
Hence if $\lambda$ is real and positive, then for all large enough $k$, we have  
\begin{align}|p_{0:k}(\lambda)| = \Omega(1/\kappa^d)\end{align}

Let us recall that, since $V$ is normal we can write
\begin{align}\|\frac{d s_{t+k}}{d h_t}\| = \sqrt{\sum_{i=1}^n|p_{0:k}(\lambda_i)|^2}\end{align}

where $\lambda_1,\ldots,\lambda_n$ are the eigenvalues of $V$.

Hence, if $V$ has at least one positive real eigenvalue then \begin{align}\|\frac{d s_{t+k}}{d h_t}\| = \Omega(1/\kappa^d)\end{align} for all large enough $k$.
\end{proof}
\noindent\rule{\textwidth}{1pt}
\begin{remark} As already mentioned, since $\kappa$ and $d$ are assumed to be constant, the theorem states that 
\begin{align}\|\frac{d s_{t+k}}{d h_t}\| = \Omega(1)\end{align}
The dependence on $\kappa$ and $d$ was simply given in order to get an intuition on how $\kappa$ and $d$ are influencing the lower bound, and that $d$ has more leverage on the lower bound than $\kappa$.

Regarding the normality of $V$, the same remark can be made as in remark \ref{normality_remark}.

Then note that if $V$ is a (real) $n\times n$ matrix, with $n$ odd, then we have at least one real eigenvalue. Thus the restriction of having at least one positive real eigenvalue is not that severe.

Further, one can show that the theorem holds in a slightly more general setting where one might not have at least one positive real eigenvalue. 

Let us consider the case where $\kappa=1$, $|\lambda|<1$ such that we could consider $\lambda^c \approx 0$ for some large enough positive integer $c$, and that all states between $T$ and $T-c$ have dependency depth of exactly $d$ (where $T=t+k$), then 

\begin{align}p_{0:k}(\lambda) = \frac{\lambda^d}{\kappa^d}\cdot (1+\lambda+\ldots+\lambda^{c-d}) = \frac{\lambda^d}{\kappa^d}\cdot \left(\frac{1-\lambda^{c-d+1}}{1-\lambda}\right)\end{align}

Hence we can see that if we can show that $\left|\frac{1-\lambda^{c-d+1}}{1-\lambda}\right|$ is lower bounded asymptotically by a constant, independent of $d$ and $\kappa$, (which it is in this case), then we have 
\begin{align}|p_{0:k}(\lambda)| = \Omega(1/\kappa^d)\end{align}
We also see that we would like $\lambda$ to be sufficiently bounded away from a small set of critical values such as the $(c-d)$-th roots of unity.\\

In a more general setting, we can rewrite 
\begin{align}p_{0:k}(\lambda) = \frac{\lambda^{d}}{\kappa^{d}}\cdot q_{0:k}(\lambda)\end{align}
for some polynomial $q_{0:k}$ with positive real coefficients, and we would like $\lambda$ to be such that $|q_{0:k}(\lambda)| = \Omega(1)$ for all sufficiently large $k$.

Our hypothesis is that the theorem holds as long as $\lambda$ is sufficiently bounded away from a small set of critical values in $\mathbb{C}\setminus \mathbb{R}^{+}$, or in other words, we would need only at least one eigenvalue to satisfy this condition. This set of critical values is a dependent on $\kappa$, $d$ and the overall configuration of the attention weights.
\end{remark}
\noindent\rule{\textwidth}{1pt}
\newpage 

\section{Effects of memory sparsity on basic reinforcement learning tasks}\label{minigrid}

We consider a few tasks from MiniGrid \cite{gym_minigrid} in the OpenAI gym \cite{gym} in which an agent must get to certain goal states. We use a partially observed formulation of the task, where the agent only sees a small number of squares ahead of it. 
Our goal is to compare generalization of the solutions learned by full and sparse memory-augmented models, by training on smaller version of an environment and testing it on a larger version. To do so, we compare the use of MemLSTM (full attention) and RelLSTM (sparse attention). We note that some purely recurrent models can perform well on these tasks where sequence lengths are rather short, but the scope of this experiment is to explicitly compare the effect of different memory densities.

\begin{table}[ht]
\centering
\small
\caption{Average Train and Test Rewards for MiniGrid Reinforcement Learning task. The models were trained on the smaller version of the environment and tested on the larger version to test to generalization of the solution learned.}
    \label{tab:rl_results}
    %\vs{-2}
\begin{tabular}{cccc}
        \toprule
         Environment & MemLSTM & RelLSTM \\
         \midrule
         & \multicolumn{2}{c}{\textbf{Train}} \\
         \midrule
         RedBlueDoors-6x6   & \bm{$0.97$} & \bm{$0.97$} \\
         GoToObject-6x6  & \bm{$0.85$} & $0.84$ \\
         MemoryS7  & $0.4$ & \bm{$0.94$}  \\
         GoToDoor-5x5  & $0.17$ & \bm{$0.25$} \\
         Fetch-5x5  & $0.42$ & \bm{$0.5$} \\
         DoorKey-5x5  & \bm{$0.94$} & $0.93$ \\
         \midrule
         & \multicolumn{2}{c}{\textbf{Test}} \\
         \midrule
         RedBlueDoors-8x8   & \bm{$0.95$} & \bm{$0.95$} \\
         GoToObject-8x8  & $0.66$ & \bm{$0.74$} \\
         MemoryS13 &  $0.24$ & \bm{$0.30$}  \\
         GoToDoor-8x8  & $0.11$ & \bm{$0.15$} \\
         Fetch-8x8 & $0.44$ & \bm{$0.45$} \\
         DoorKey-16x16  & $0.31$ & \bm{$0.44$} \\
         \bottomrule
    \end{tabular}
\end{table}

These tasks are difficult to solve with standard RL algorithms, due to (1) the partial observability of the environment and (2) the sparsity of the reward, given that the agent receives a reward only after reaching the goal. We use Proximal Policy Optimization (PPO, \cite{PPO}) along with MemLSTM, and RelLSTM as the recurrent modules. All models were each trained for $5000000$ steps on each environment. The hyperparameters used for RelLSTM are $\nu = 5$ and $\rho = 5$.
On the \textit{MiniGrid-DoorKey-5x5-v0} environment the average reward for MemLSTM is $0.94$ and RelLSTM is $0.93$. On transferring the learned solution to the \textit{16x16} version of that environment the average reward for MemLSTM is $0.31$ and RelLSTM is \bm{$0.44$}. As illustrated in \ref{tab:rl_results}, we find that transfer scores for RelLSTM are much higher than for MemLSTM across several environments.

% \section{MiniGrid RL task} \label{minigrid}
% For the MiniGrid environments, see the official repository \url{https://github.com/maximecb/gym-minigrid}. The models were each trained for $5000000$ steps on each environment. The hyperparameters used for RelLSTM are $\nu = 5$ and $\rho = 5$. 
% \begin{table}[h]
% \centering
% \caption{Average Train and Test Rewards for MiniGrid Reinforcement Learning task. The models were trained on the smaller version of the environment and tested on the larger version to test to generalization of the solution learned.}
%     \label{tab:rl_results}
% \begin{tabular}{cccc}
%         \toprule
%          Environment & LSTM & MemLSTM & RelLSTM \\
%          \midrule
%          & \multicolumn{3}{c}{\textbf{Train}} \\
%          \midrule
%          RedBlueDoors-6x6  & \bm{$0.97$} & \bm{$0.97$} & \bm{$0.97$} \\
%          GoToObject-6x6 & $0.81$ & \bm{$0.85$} & $0.84$ \\
%          MemoryS7 & \bm{$0.96$} & $0.4$ & $0.94$  \\
%          GoToDoor-5x5 & \bm{$0.28$} & $0.17$ & $0.25$ \\
%          Fetch-5x5 & $0.48$ & $0.42$ & \bm{$0.5$} \\
%          DoorKey-5x5 & \bm{$0.94$} & \bm{$0.94$} & $0.93$ \\
%          \midrule
%          & \multicolumn{3}{c}{\textbf{Test}} \\
%          \midrule
%          RedBlueDoors-8x8  & \bm{$0.95$} & \bm{$0.95$} & \bm{$0.95$} \\
%          GoToObject-8x8 & $0.66$ & $0.66$ & \bm{$0.74$} \\
%          MemoryS13 & $0.25$ & $0.24$ & \bm{$0.30$}  \\
%          GoToDoor-8x8 & $0.13$ & $0.11$ & \bm{$0.15$} \\
%          Fetch-8x8 & $0.38$ & $0.44$ & \bm{$0.45$} \\
%          DoorKey-16x16 & $0.09$ & $0.31$ & \bm{$0.44$} \\
%          \bottomrule
%     \end{tabular}
% \end{table}
\newpage 
\section{Tradeoff analysis between sparsity and gradient propagation}
\label{tradeoff_results}
As already discussed in Section \ref{heuristics}, the sparsity coefficient $\kappa$ verifies $\kappa = \nu + \rho \geq |S_t| + |R_t|$ for all time step $t$, where we denote $\nu$ for the size of the short-term buffer, and $\rho$ for the maximal size of the relevant sets $R_t$. In this section we would like to see how gradient propagation varies when changing sparsity. As already discussed at the end of Section \ref{Theoretical analysis} as well as at the end of Section \ref{heuristics}, decreasing $\kappa$, would increasingly force gradients to backpropagate through the recurrent connections, thus degrading gradient stability. Meanwhile, increasing $\kappa$ would increase the size of the computational graph. Thus we would like to find the optimal trade-off between sparsity and gradient propagation. This trade-off is clearly task-specific and needs to be determined experimentally. The only way to do so is by either changing $\nu$ or changing $\rho$ (or both). Hence we are going to analyze the effects on gradient propagation by separately changing $\nu$ and $\rho$. 

\begin{figure}[H]
    \centering
   \begin{subfigure}%[t]{0.35\textwidth}
   \centering
    \includegraphics[width=2.5in,  height=1.7in]{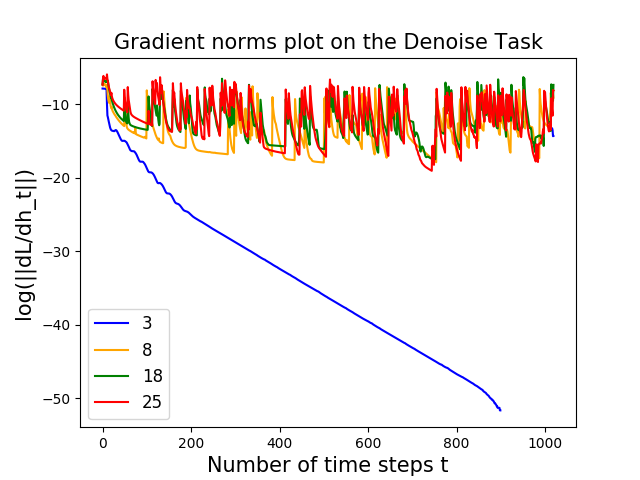}
    \end{subfigure}
    \begin{subfigure}%[t]{0.35\textwidth}
    \centering
    \includegraphics[width=2.5in,  height=1.7in]{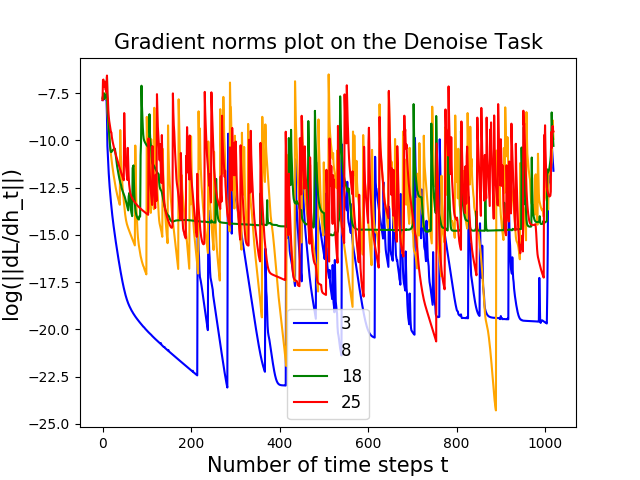}
    \end{subfigure}
    %\vs{1}
    \caption{Both sides show gradient norm plots of $\|\nabla_{h_t}L\|$ in log scale after training for Denoise Task with $t$ ranging from 0 (latest time step) to 1000 (furthest time step). \textbf{(Left)} We took four MemLSTM models for $\rho=3,8,18,25$ while keeping $\nu=15$ fixed. \textbf{(Right)} We took four MemLSTM models for $\nu=3,8,18,25$ while keeping $\rho=15$ fixed. (Note that the $y$-axis of the two plots have different scales, as indicated in the plots.) }\label{fig:trade_off}
    %\vs{4}
\end{figure}

For Figure \ref{fig:trade_off} (left), we can see that when choosing $\rho$ too small (here for instance $\rho=3$), gradient propagation becomes unstable, while larger values for $\rho$ all show stable gradient propagation. This confirms our initial intuition that we can decrease $\rho$ until a task-specific treshold and maintain stable gradient propagation, while decreasing $\rho$ beyond this treshold would cause gradient propagation to become unstable. 

For Figure \ref{fig:trade_off} (right), we can see that changing $\nu$ has much less leverage on gradient propagation than changing $\rho$. Gradient propagation stays relatively stable regardless of the values for $\nu$. The only difference is that for the extreme value of $\nu=3$, we can see that gradient propagation became slightly less stable, because with smaller $\nu$ predictions for future relevancy might become less accurate.
\newpage 
\section{Additional Results}
\label{extra_results}

\begin{figure*}[ht!]
    \centering
    %\vs{1}
    \begin{subfigure}
        \centering
        \includegraphics[scale=0.25]{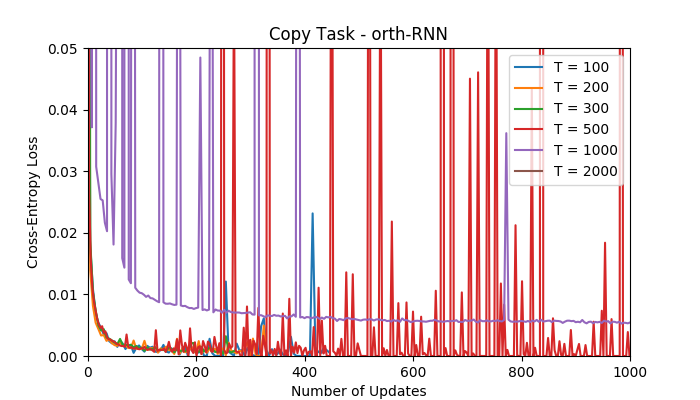}
    \end{subfigure}
    \centering
    \begin{subfigure}
        \centering
        \includegraphics[scale=0.25]{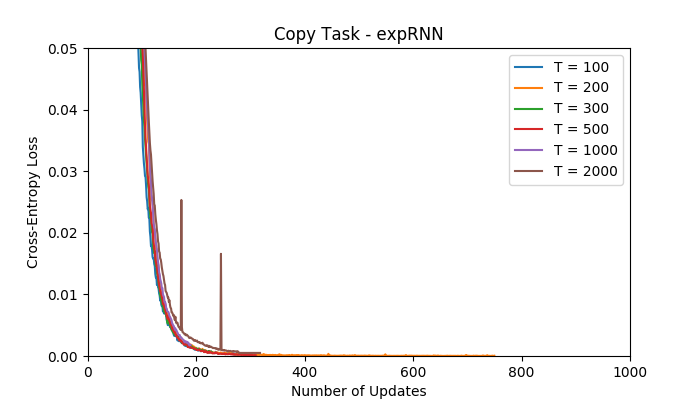}
    \end{subfigure}
    \centering
    \begin{subfigure}
        \centering
        \includegraphics[scale=0.25]{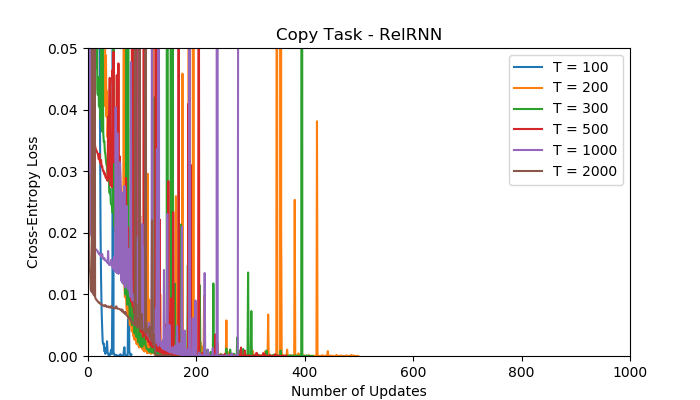}
    \end{subfigure}
    \centering
    \begin{subfigure}
        \centering
        \includegraphics[scale=0.25]{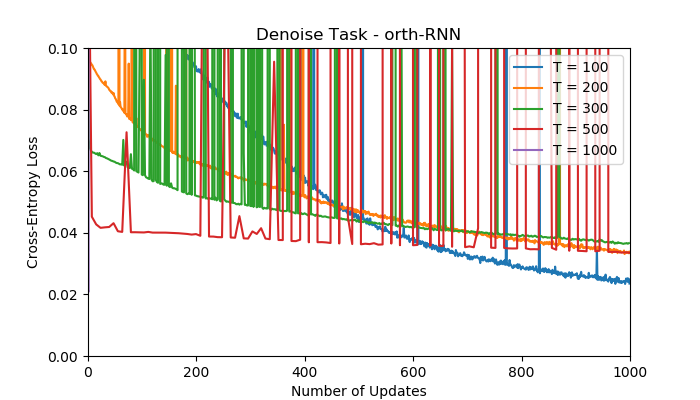}
    \end{subfigure}
    \begin{subfigure}
        \centering
        \includegraphics[scale=0.25]{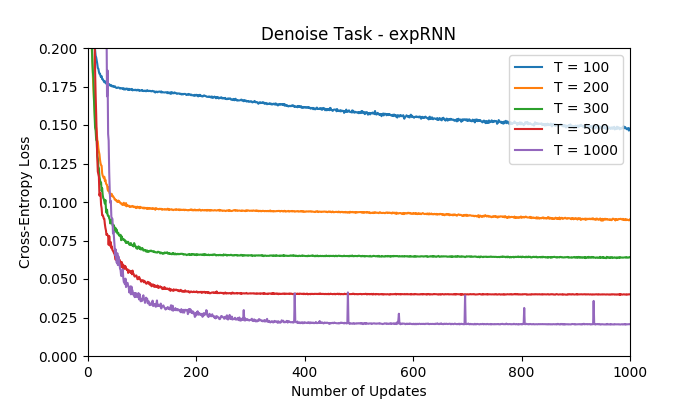}
    \end{subfigure}
    \begin{subfigure}
        \centering
        \includegraphics[scale=0.25]{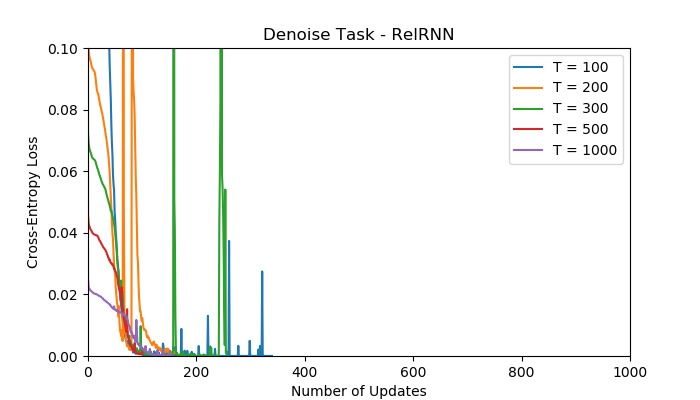}
    \end{subfigure}
    %\vs{2}
    \caption{Cross-entropy vs training updates for Copy (top) and Denoise (bottom) tasks for $T = \{100, 200, 300, 500, 1000, 2000 \}$.  $1$ unit of the x-axis is equal to $100$ iterations of training with the exception of expRNN where 1 unit on the x-axis is $10$ iterations of training.}
    \label{fig:exp}
    %\vs{3}
\end{figure*}

\begin{table}[ht]
    \centering
    \caption{Results for Copy Task}
    \begin{tabular}{cccccccc}
         \toprule
         $T$ & LSTM & orth-RNN & expRNN & MemRNN & SAB & RelRNN & RelLSTM \\
         \midrule
         $100$ & $100\%$ & $100\%$ & $100\%$ & $100\%$ & $100\%$ & $100\%$ & $100\%$ \\
         $200$ & $100\%$ & $100\%$ & $100\%$ & $100\%$ & $100\%$ & $100\%$ & $100\%$ \\
         $300$ & $100\%$ & $100\%$ & $100\%$ & $100\%$ & $100\%$ & $100\%$ & $100\%$\\
         $500$ & $12\%$ & $100\%$ & $100\%$ & $100\%$ & $100\%$ & $100\%$ & $100\%$ \\
         $1000$ & $12\%$ & $80\%$ & $100\%$ & $100\%$ & $100\%$ & $100\%$ & $100\%$ \\
         $2000$ & $12\%$ & $11\%$ & $100\%$ & OOM & $100\%$ & $100\%$ & $100\%$ \\
         \bottomrule
    \end{tabular}
    \label{tab:copy_table}
\end{table}

\begin{table}[ht]
    \centering
    \caption{Hyperparameters used for Copy task}
    \begin{tabular}{cccccc}
         \toprule
         Model & lr & optimizer& non-linearity & $\nu$ & $\rho$ \\
         \midrule
         orthRNN & $0.0002$ & RMSprop & modrelu & - & - \\
         expRNN & $0.0002$  & RMSprop & modrelu  & - & - \\
         LSTM & $0.0002$ & Adam & - & - & - \\
         RelRNN & $0.0002$ & Adam & tanh & $10$ & $10$ \\
         \bottomrule
    \end{tabular}
    \label{tab:hp_copy}
\end{table}

\begin{table}[ht]
    \centering
    \caption{Hyperparameters used for Denoise task}
    \begin{tabular}{cccccc}
         \toprule
         Model & lr & optimizer & non-linearity & $\nu$ & $\rho$ \\
         \midrule
         orthRNN & $0.0002$ & RMSprop & modrelu & - & - \\
         expRNN & $0.0002$  & RMSprop & modrelu  & - & - \\
         LSTM & $0.0002$ & Adam & - & - & - \\
         GORU & $0.001$ & RMSprop & - & - & - \\ 
         RelRNN & $0.0002$ & RMSprop & modrelu & $10$ & $10$ \\
         \bottomrule
    \end{tabular}
    \label{tab:hp_denoise}
\end{table}

\begin{table}[ht]
    \centering
    \caption{Hyperparameters used for sequential MNIST}
    \begin{tabular}{cccccc}
         \toprule
         Model & lr (lr orth)  & optimizer & non-linearity & $\nu$ & $\rho$  \\
         \midrule
         orthRNN & $0.0001$ & Adam & modrelu & - & - \\
         expRNN & $0.0001 (0.00001)$ & Adam  & modrelu  & - & - \\
         LSTM & $0.0002$ &  & - & - & - \\
         GORU & & &  - & - \\
         RelRNN & $0.0003$ & Adam & modrelu & $10$ & $10$ \\
         \bottomrule
    \end{tabular}
    \label{tab:smnist_table}
\end{table}

\begin{table}[ht]
    \centering
    \caption{Hyperparameters used for PTB}
    \begin{tabular}{cccccc}
         \toprule
         Model & lr (lr orth) & optimizer & non-linearity & $\nu$ & $\rho$ \\
         \midrule
         orthRNN & $0.001$ & Adam & tanh & - & - \\
         expRNN & $0.003 (0.0003)$ & Adam  & tanh  & - & - \\
         LSTM & $0.0002$ &  & - & - & - \\
         GORU & &  & - & - \\
         RelRNN & $0.0003$ & Adam  & tanh & $10$ & $5$ \\
         \bottomrule
    \end{tabular}
    \label{tab:smnist_table}
\end{table}

%We also ran all tasks for the Transformer model. We find that it is difficult to get a transformer to learn on the Copy and Denoise tasks as its performance is no better than random guessing.

\begin{figure}
    \centering
    \includegraphics[scale=0.5]{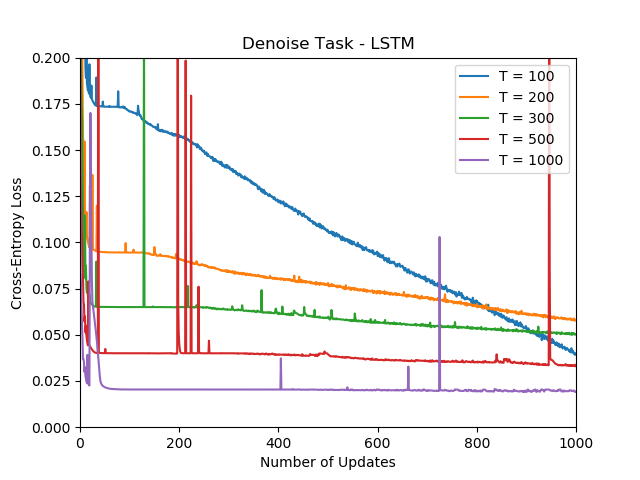}
    \caption{Training curves for LSTM on Denoise task}
    \label{fig:my_label}
\end{figure}

\begin{figure}
    \centering
    \includegraphics[scale=0.5]{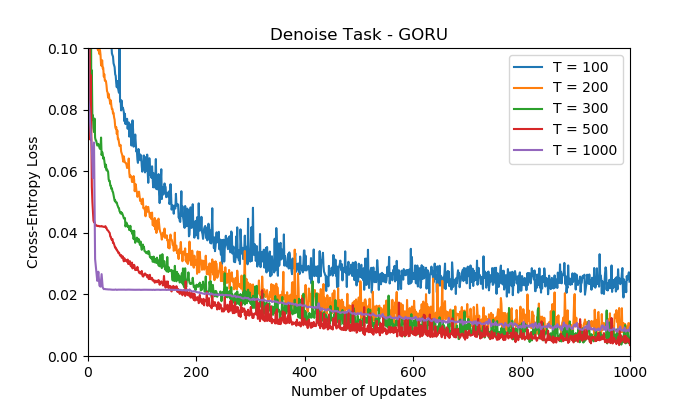}
    \caption{Training curves for GORU on Denoise task}
    \label{fig:my_label}
\end{figure}

\begin{figure}
    \centering
    \includegraphics[scale=0.5]{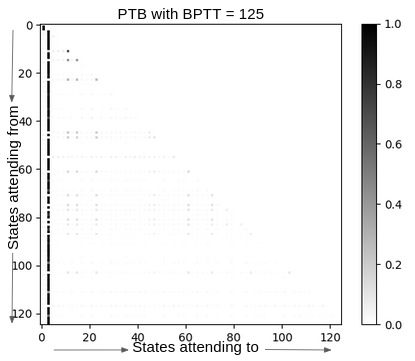}
    \caption{Heatmap of attention scores on PTB task training with full attention and BPTT of $125$}
    \label{fig:my_label}
\end{figure}

\begin{figure}
    \centering
    \includegraphics[scale=0.5]{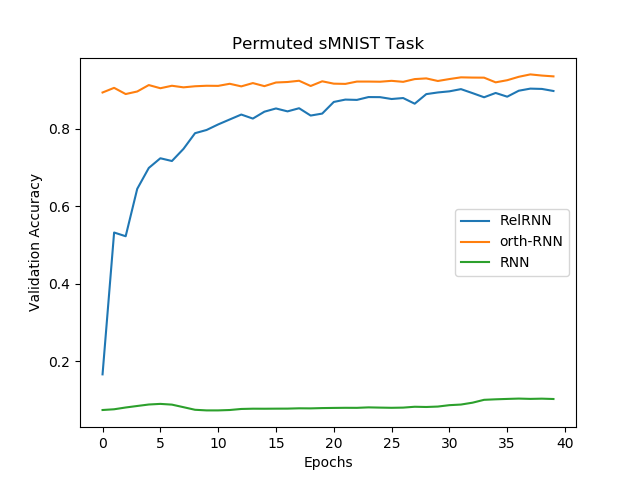}
    \caption{Validation accuracy curves for pMNIST}
    \label{fig:my_label}
\end{figure}

\begin{figure}
    \centering
    \includegraphics[scale=0.5]{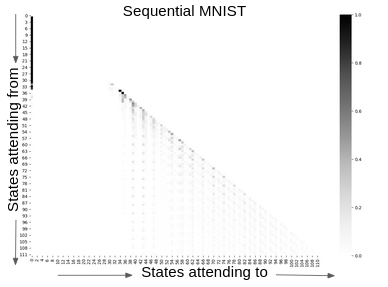}
    \caption{Heatmap of attention scores on MNIST digit classification. 7 pixels were grouped at each time step to make visualization of heatmap easier.}
    \label{fig:my_label}
\end{figure}

\begin{figure}
    \centering
    \includegraphics[scale=0.6]{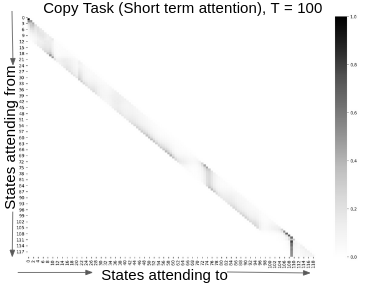}
    \caption{Heatmap of attention scores on Copy task when only doing attention over the Short Term Buffer.}
    \label{fig:my_label}
\end{figure}

\begin{figure}
    \centering
    \includegraphics[scale=0.6]{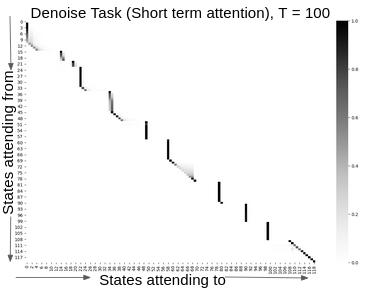}
    \caption{Heatmap of attention scores on Denoise task when only doing attention over the Short Term Buffer.}
    \label{fig:my_label}
\end{figure}

%This corresponds to a prior on the kind of dependencies which humans can consciously reason with: as proposed with the consciousness prior~\cite{bengio2017consciousness}, the joint dependency between all the high-level variables can be decomposed into factors each involving only a few variables or events (which can all fit in short-term memory).
%\begin{table*}
%\centering
%\begin{tabular}{c|cccccc}
%Model & hid & LR (LR orth) & $\alpha$ &  \\
%LSTM &
%\end{tabular}
%\caption{Hyperparameters for the Penn Tree Bank task. Here, "hid" is hidden state size, "LR" is learning rate, "LR orth" is the learning rate of the orthogonal recurrent weight matrix in expRNN, $\alpha$ is the smoothing parameter of RMSprop.}
%\label{tab:hp_ptb}
%\end{table*}

\end{document}